\documentclass{article}

\usepackage[final,nonatbib]{neurips_2019}

\usepackage[utf8]{inputenc} % allow utf-8 input
\usepackage[T1]{fontenc}    % use 8-bit T1 fonts
\usepackage[colorlinks = true,
            linkcolor = blue,
            urlcolor  = blue,
            citecolor = blue,
            anchorcolor = blue]{hyperref}
\usepackage{url}            % simple URL typesetting
\usepackage{booktabs}       % professional-quality tables
\usepackage{amsfonts}       % blackboard math symbols
\usepackage{nicefrac}       % compact symbols for 1/2, etc.
\usepackage{microtype}      % microtypography

\usepackage{algorithm}
\usepackage{algorithmic}
\usepackage{eqparbox}
\renewcommand{\algorithmiccomment}[1]{\hfill\eqparbox{COMMENT}{\%~#1}}

\usepackage{verbatim}
\usepackage{amsthm}
\usepackage{amsfonts}
\usepackage{amsmath}
\usepackage{amssymb}
\usepackage{sansmath}
\usepackage{color}
\usepackage{xcolor}
\usepackage{bm}
\usepackage{bbm}
\usepackage{dsfont}
\newtheorem{theorem}{Theorem}
\newtheorem{definition}[theorem]{Definition}

\newtheorem{proposition}[theorem]{Proposition}
\newtheorem{lemma}[theorem]{Lemma}
\newtheorem{corollary}[theorem]{Corollary}
\newtheorem{claim}[theorem]{Claim}

\usepackage{subfig}
\usepackage{graphicx}
\usepackage{listings}
\usepackage{enumitem}
\newtheoremstyle{named}{}{}{\itshape}{}{\bfseries}{.}{.5em}{\thmnote{#3}}
\theoremstyle{named}
\newtheorem*{namedtheorem}{Theorem}
\usepackage{tikz}
\usetikzlibrary{fit,calc}

\DeclareSymbolFont{extraup}{U}{zavm}{m}{n}
\DeclareMathSymbol{\varheart}{\mathalpha}{extraup}{86}
\DeclareMathSymbol{\vardiamond}{\mathalpha}{extraup}{87}
\DeclareMathOperator*{\argmax}{arg\,max}

\newcommand{\RR}{\mathbb{R}}
\newcommand{\PP}{\mathbb{P}}
\newcommand{\EE}{\mathbb{E}}
\newcommand{\ee}{\mathbf{e}}
\newcommand{\II}{\mathbb{I}}

\newcommand{\NN}{\mathcal{N}}

\newcommand{\HH}{\mathcal{H}}
\newcommand{\GG}{\mathcal{G}}
\newcommand{\OO}{\mathcal{O}}

\newcommand{\cS}{\mathcal{S}}
\newcommand{\cA}{\mathcal{A}}
\newcommand{\cT}{\mathcal{T}}
\newcommand{\cD}{\mathcal{D}}
\newcommand{\cM}{\mathcal{M}}
\newcommand{\cU}{\mathcal{U}}
\newcommand{\cZ}{\mathcal{Z}}

\DeclareMathOperator\Var{\mathrm{Var}}

\makeatletter
\newcommand*{\inlineequation}[2][]{%
  \begingroup
    \refstepcounter{equation}%
    \ifx\\#1\\%
    \else
      \label{#1}%
    \fi
    \relpenalty=10000 %
    \binoppenalty=10000 %
    \ensuremath{%
      #2%
    }%
    ~\@eqnnum
  \endgroup
}
\makeatother
\title{Privacy-preserving Q-Learning with Functional Noise in Continuous Spaces}

\author{%
  Baoxiang Wang \\
  Borealis AI, Edmonton, Canada\\
  \texttt{bxiangwang@gmail.com} \\
  \And
  Nidhi Hegde \\
  Borealis AI, Edmonton, Canada \\
  \texttt{nidhi.hegde@borealisai.com}
}

\begin{document}

\maketitle

\begin{abstract}
We consider differentially private algorithms for reinforcement learning in continuous spaces, such that neighboring reward functions are indistinguishable. This protects the reward information from being exploited by methods such as inverse reinforcement learning. Existing studies that guarantee differential privacy are not extendable to infinite state spaces, as the noise level to ensure privacy will scale accordingly to infinity. Our aim is to protect the value function approximator, without regard to the number of states queried to the function. It is achieved by adding functional noise to the value function iteratively in the training. We show rigorous privacy guarantees by a series of analyses on the kernel of the noise space, the probabilistic bound of such noise samples, and the composition over the iterations. We gain insight into the utility analysis by proving the algorithm's approximate optimality when the state space is discrete. Experiments corroborate our theoretical findings and show improvement over existing approaches.
\end{abstract}

\section{Introduction}
\label{sec:intro}

Increasing interest in reinforcement learning (RL) and deep reinforcement learning has led to recent advances in a wide range of algorithms \cite{sutton2018reinforcement}.
While a large part of the advancement has been in the space of games, the applicability of RL extends to other practical cases such as recommendation systems \cite{zheng2018drn,liebman2015dj} and search engines \cite{rosset2018optimizing,hu2018reinforcement}.
With the popularity of the RL algorithms increasing, so have concerns about their privacy.
Namely, the released value (or policy) function are trained based on the reward signal and other inputs, which commonly rely on sensitive data.
For example, an RL recommendation system may use the reward signals simulated by users' historical records.
This historical information can thus be inferred by recursively querying the released functions.
We consider differentially privacy \cite{dwork2006calibrating,dwork2014algorithmic}, a natural and standard privacy notion, to protect such information in the RL methods.

RL methods learn by carrying out actions, receiving rewards observed for that action in a given state, and transitioning to the next states.
Observation of the learned value function can reveal sensitive information: \textit{the reward function} is a succinct description of the task. It is also connected to the users' preferences and the criteria of their decision-making; \textit{the visited states} carry important contextual information on the users, such as age, gender, occupation, and etc.; \textit{the transition function} includes the dynamics of the system and the impact of the actions on the environment.
Among those, the reward function is the most vulnerable and valuable component, and studies have been conducted to infer this information \cite{abbeel2004apprenticeship,ng2000algorithms}.
In this paper, our aim is to design differentially private algorithms for RL, such that neighboring reward functions are indistinguishable.

There is a recent line of research on privacy-preserving algorithms by protecting the reward function. Balle et al.~\cite{balle2016differentially} train the private value function using a fixed set of trajectories. However when a new state is queried this privacy guarantee will not hold. Similar results are also considered in contextual multi-arm bandits \cite{shariff2018differentially,sajed2019optimal,pan2018reinforcement}, where the context vector is analogous to the state. The gap that these works leave lead us to design a private algorithm that is not dependent on the number of states queried to the value function.

In order to achieve this under continuous space settings, we investigate the Gaussian process mechanism proposed by Hall et al.~\cite{hall2013differential}. The mechanism adds functional noise to the value function approximation hence the function can be evaluated at arbitrarily many states while preserving privacy. We show that our choice of the reproducing kernel Hilbert space (RKHS) embeds common neural networks, hence a nonlinear value function can also be used.
We therefore adapt Q-learning  \cite{mnih2015human,watkins1992q,baird1995residual} so that the value function is protected after each update, even when new states are visited. 

We rigorously show differential privacy guarantees of our algorithm with a series of techniques.
Notably, we derive a probabilistic bound of the sample paths thus ensuring that the RKHS norm of the noised function can be bounded. 
This bound is significantly better than a union bound of all noise samples. 
Further, we analyze the composition of the privacy costs of the mechanism. 
There is no known composition result of the functional mechanism, other than the general theorems that apply to any mechanism \cite{kairouz2013composition,dwork2010boosting,beimel2010bounds}.
Inspired by these theorems, we derive a privacy guarantee which is better than existing results. 
On the utility analysis, though there is no known performance analysis on deep reinforcement learning, we gain insights by proving the utility guarantee under the tractable discrete state space settings. Empirically, experiments corroborate our theoretical findings and show improvement over existing methods.

\textbf{Related Works.} There is a recent line of research that discusses privacy-preserving approaches on online learning and stochastic multi-armed bandit problems \cite{sutton2018reinforcement,szepesvari2010algorithms}. 
The algorithms protect neighboring reward sequences from being distinguished, which is related to our definition of neighboring reward functions.
In bandit problems, the algorithms preserve the privacy via mechanisms that add noise to the estimates of the reward distribution \cite{tossou2017achieving,tossou2016algorithms,mishra2015nearly,thakurta2013nearly,kusner2015differentially}.  This line of work shares similar motivations as our work, but they do not scale to the continuous space because of the $\sqrt{N}$ or $\sqrt{N\ln N}$ factor involved where $N$ is the number of arms.
Similarly, in the online learning settings, the algorithms preserve the privacy evaluated sequence of the oracle \cite{gajane2017corrupt,abernethy2017online,agarwal2017price,jain2012differentially}. Their analyses are based on optimizing a fixed objective thus do not apply to our setting.

More closely related are privacy studies on contextual bandits \cite{sajed2019optimal,shariff2018differentially}, where there is a contextual vector that is analogous to the states in reinforcement learning. Equivalently, differentially private policy evaluation \cite{balle2016differentially} considers a similar setting where the value function is learned on a one-step MDP. It worth note that they also consider the privacy with respect to the state and the actions, though in this paper we will focus only on the rewards.
The major challenge to extend these works is that reinforcement learning requires an iterative process of policy evaluation and policy improvement. 
The additional states that are queried to the value function are not guaranteed to be visited and protected by previous iterations. 
We propose an approach for both the evaluation and the improvement, while also extending the algorithm to nonlinear approximations like neural networks. 

Differential privacy in a Markov decision process (MDP) has been discussed \cite{venkitasubramaniam2013privacy} via the input perturbation technique. In the work, 
the reward is reformulated as a weighted sum of the utility and the privacy measure. With this formulation, it amounts to learn the MDP under this weighted reward. Essentially, input perturbation will cause relatively large utility loss and is therefore less preferred. Similarly, output perturbation can be used to preserve privacy, as shown in our analysis. It is though obvious that the necessary noise level is relatively larger and also depends on more factors than our algorithm does. Therefore, more subtle techniques will be required to improve the methods by input and output perturbation.

A general approach that can be applied to continuous spaces is the differentially private deep learning framework \cite{abadi2016deep,chamikara2019local}. The method perturbs the gradient estimator in the updates of the neural network parameters to preserve privacy.
In our problem, applying the method will require large noise levels.  
In fact, the algorithm considers neighboring inputs that at most one data point can be different, therefore benefits from a $1/B$ factor via privacy amplification \cite{kasiviswanathan2011can,beimel2010bounds} where $B$ is the batch size.
This no longer holds in reinforcement learning, as all reward signals can be different for neighboring reward functions, causing the noise level to scale $B$ times back.

\section{Preliminaries}

\subsection{Markov Decision Process and Reinforcement Learning}
\label{sec:pre-rl}

Markov decision process (MDP) is a framework to model decisions in an environment.
We use canonical settings of the discrete-time Markov decision process. An MDP is 
denoted by the tuple $(\cS,\cA, \cT, r,\rho_0,\gamma)$ which includes the state space $\cS$, the action space $\cA=\{1,\dots,m\}$, the stochastic transition kernel $\cT:\cS\times\cA\times\cS \to \RR^+$, the reward function $r:\cS\times\cA\to\RR$, the initial state distribution $\rho_0:\cS\to \RR^+$ and the discount factor $\gamma \in [0,1)$. Denote $m$ in the above as the number of actions in the action space. The objective is to maximize the expected discounted  cumulative reward. Further define the policy function $\pi:\cS,\cA\to \RR^+$ and the corresponding action-state value function as
\[
Q^\pi(s,a) = \EE_{\pi}[\sum_{t\geq 0}^\infty \gamma^{t}r(s_{t},a_{t})|s_0=s,a_0=a,\pi].
\]
When the context is clear, we omit $\pi$ and write $Q(s,a)$ instead. 

We use the continuous state space setting for this paper, except in Appendix \ref{sec:utility}. 
We investigate bounded and continuous state space $\cS\subseteq \RR$ and without loss of generality assume that $\cS= [0,1]$. 
The value function $Q(s,a)$ is treated as a set of $m$ functions $Q_a(s)$, where each function is defined on $[0,1]$. 
The reward function is similarly written as a set of $m$ functions, each defined on $[0,1]$.  
We do not impose any particular assumptions on the reward function. 

Our algorithm is based on deep Q-learning \cite{mnih2015human,baird1995residual,watkins1992q}, which solves the Bellman equation. Our differential privacy guarantee can also be generalized to other Q-learning algorithms. The objective of deep Q-learning is to minimize the Bellman error
\[
\frac{1}{2}(Q(s,a) - \EE[r + \gamma\max_{a^\prime} Q(s^\prime,a^\prime)])^2,
\]
where $s^\prime\sim \cT(s, a, s^\prime)$ denotes the consecutive state after executing action $a$ at state $s$. Similar to \cite{mnih2015human}, we use a neural network to parametrize $Q(s,a)$. We will focus on output a learned value function where the reward function $r(\cdot)$ and $r^\prime(\cdot)$ cannot be distinguished by observing $Q(s,a)$, as long as $\|r-r^\prime\|_\infty\leq 1$. Here without ambiguity we write $r(\cdot), r^\prime(\cdot)$ as $r,r^\prime$, and the infinity norm $\|f(s)\|_\infty$ is defined as $\sup_s|f(s)|$.

\subsection{Differential Privacy}
\label{sec:pre-dp}

Differential privacy~\cite{dwork2006our,dwork2006calibrating} has developed into a strong standard for privacy guarantees in data analysis. It provides a rigorous framework for privacy guarantees under various adversarial attacks.  

The definition of differential privacy is based on the notion that in order to preserve privacy, data analysis should not differ at the aggregate level whether any given user is present in the input or not. This latter condition on the presence of any user is formalized through the notion of neighboring inputs. The definition of neighboring inputs will vary according to the problem settings.

Let $d,d^\prime \in \cD$ be neighboring inputs.
\begin{definition}
A randomized mechanism $\cM : \cD \rightarrow \cU$ satisfies $(\epsilon,\delta)$-differential privacy if for any two neighboring inputs $d$ and $d^\prime$ and for any subset of outputs $\cZ \subseteq \cU$ it holds that
\[
\PP(\cM(d) \in \cZ) \le \exp(\epsilon)\PP(\cM(d^\prime) \in \cZ) + \delta.
\]
\end{definition}

An important parameter of a mechanism is the (global) sensitivity of the output.

\begin{definition}
For all pairs $d,d^\prime \in \cD$ of neighboring inputs, the sensitivity of a mechanism $\cM$ is defined as
\begin{equation}
\label{eq:l2sens}
\Delta_{\cM} = \sup_{d,d^\prime \in \cD}\|\cM(d)-\cM(d^\prime)\|,
\end{equation}
\end{definition}
where $\|\cdot\|$ is a norm function defined on $\cU$.

\textbf{Vector-output mechanisms.}
For converting vector-valued functions into a $(\epsilon,\delta)$-DP mechanism, one of the standard approaches is the Gaussian mechanism. This mechanism adds $\NN(0,\sigma^2\II)$ to the output $\cM(d)$. In this case $\cU=\RR^n$ and $\|\cdot\|$ in \eqref{eq:l2sens} is the $\ell^2$-norm $\|\cdot\|_2$.
\begin{proposition}[Vector-output Gaussian mechanism; Theorem A.1 of \cite{dwork2014algorithmic}]
\label{gaussianmechanism}
If $0<\epsilon<1$ and $\sigma\geq \sqrt{2\ln(1.25/\delta)}\Delta_{\cM}/\epsilon$, then $\cM(d)+y$ is ($\epsilon$,$\delta$)-differentially private, where $y$ is drawn from $\NN(0,\sigma^2 I)$.
\end{proposition}

\textbf{Function-output mechanisms.}
In this setting the output of the function is a function, which means the mechanism is a functional. We consider the case where $\cU$ is an RKHS and $\|\cdot\|$ in \eqref{eq:l2sens} is the RKHS norm $\|\cdot\|_\HH$. Hall et al. \cite{hall2013differential} have shown that
adding a Gaussian process noise $\GG(0,\sigma^2K)$ to the output $\cM(d)$ is differentially private, when $K$ is the RKHS kernel of $\cU$. Let $\GG$ denote the Gaussian process distribution.
\begin{proposition}[Function-output Gaussian process mechanism \cite{hall2013differential}]
\label{clm:hall}
If $0<\epsilon<1$ and $\sigma\geq \sqrt{2\ln(1.25/\delta)}\Delta_{\cM}/\epsilon$, then $\cM(d)+g$ is ($\epsilon$,$\delta$)-differentially private, where $g$ is drawn from $\GG(0,\sigma^2 K)$ and $\cU$ is an RKHS with kernel function $K$. 
\end{proposition}

Note that in \cite{hall2013differential} the stated condition was $\sigma\geq \sqrt{2\ln(2/\delta)}\Delta_{\cM}/\epsilon$. The improvement from constant $2$ to $1.25$ is natural but for the completeness we include a proof in Appendix \ref{appendix:others}.

\section{Differentially Private Q-Learning}
\label{sec:dpql}

%We apply privacy to a setting that can be generalized to a variety of learning tasks, the Q-learning framework of reinforcement learning, where the objective is to maximize the action-value function.  We use function approximators (i.e. a neural network) parameterized by $\theta$ to learn the optimal action-value function.   We consider the continuous state space setting, where $Q(s,a)$ is assumed to a set of $m$ functions $Q_a(s)$ defined on $[0,1]$ and similarly the reward is a set of $m$ functions each defined on $[0,1]$.  

\subsection{Our Algorithm}

We present our algorithm for privacy-preserving Q-learning under the setting of continuous state space in Algorithm~\ref{alg:dpql}. The algorithm is based on deep Q-learning proposed by Mnih et al.~\cite{mnih2015human}. We achieve privacy by perturbing the learned value function at each iteration, by adding a Gaussian process noise. The noise-adding is described by line $19$-$20$ of the algorithm, where $\hat g$ is the noise. This noise is a discrete estimate of the continuous sample path, evaluated at the states $s_t$ visited in the trajectories. Intuitively, when $(s,z)$ is an element of the list $\hat g$, it implies $g(s)=z$ for the sample path $g$. Line $14$-$18$ describes the necessary maintenance of $\hat g$ to simulate the Gaussian process. Line $7$-$9$ samples a new Gaussian process sample path for every $J$ iterations, which controls the balance between the approximation factor of privacy and the utility. The other steps are similar to~\cite{mnih2015human}.

\begin{algorithm}[ht!]
   \caption{Differentially Private Q-Learning with Functional Noise}
   \label{alg:dpql}
\begin{algorithmic}[1]
   \STATE {\bfseries Input:} the environment and the reward function $r(\cdot)$
   \STATE {\bfseries Parameters:} target privacy $(\epsilon, \delta)$, time horizon $T$, batch size $B$, action space size $m$, learning rate $\alpha$, reset factor $J$
   \STATE {\bfseries Output:} trained value function $Q_\theta(s,a)$
   \STATE {\bfseries Initialization:} $s_0\in [0,1]$ uniformly, $Q_\theta(s,a)$ for each $a\in [m]$, linked list $\hat{g}_k[B][2]=\{\}$
   \STATE Compute noise level $\sigma = \sqrt{2(T/B)\ln(e+\epsilon/\delta)}C(\alpha, k, L, B)/\epsilon$;
   \FOR{$j$ {\bfseries in} $[T/B]$}
   \IF {$ j \equiv 0 \mod{T/JB}$}
   \STATE $\hat{g}_k[B][2]\leftarrow\{\}$;
   \ENDIF
   \FOR{$b$ {\bfseries in} $[B]$}
   \STATE $t\leftarrow jT/B+b$;
   \STATE Execute $a_t=\argmax_a Q_\theta(s_t,a)+\hat{g}_a(s_t)$;
   \STATE Receive $r_t$ and $s_{t+1}$, $s\leftarrow s_{t+1}$;
   \FOR{$a\in [m]$}
   \STATE Insert $s$ to $\hat{g}_a[:][1]$ such that the list remains monotonically increasing;
   \STATE Sample $z_{at}\sim \NN(\mu_{at}, \sigma d_{at}))$, according to Equation \eqref{eqn:efficient}, Appendix \ref{appendix:proofpos};
   \STATE Update the list $\hat{g}_a(s) \leftarrow z_{at}$;
   \ENDFOR
   \STATE $y_t\leftarrow r_t+\gamma\max_a Q_\theta(s_{t+1}, a)+\hat{g}_a(s_{t+1})$;
   \STATE $l_t\leftarrow \frac{1}{2}(Q_\theta(s_t,a_t)+\hat{g}_a(s_{t})-y_t)^2$;
   \ENDFOR
   \STATE Run one step SGD $\theta\leftarrow \theta+\alpha\frac{1}{B}\nabla_\theta \sum_{t=jB}^{(j+1)B}l_t$;
   \ENDFOR
\STATE Return the trained $Q_\theta(s, a)$ function;
\end{algorithmic}
\end{algorithm}

\textbf{Insight into the algorithm design.} To satisfy differential privacy guarantees, we require two reward functions $r$ and $r^\prime$ to be indistinguishable upon observation of the learned functions, as long as $\|r-r^\prime\|_\infty\leq 1$. The major difficulty is that the reward signal $r(s,a)$ can appear at any $s$, and all the reward signals can be different under $r$ and $r^\prime$. Therefore, we will need a stronger mechanism of privacy that does not rely on the finite setting where at most one data point in a (finite) dataset is different, like in \cite{abadi2016deep} and \cite{balle2016differentially}. This is also the major challenge in extending~\cite{balle2016differentially} from policy evaluation to policy improvement. The natural approach to address the challenge is to treat a function as one ``data point'', which leads to our utilization of the techniques studied by Hall et al.~\cite{hall2013differential}.

\subsection{Privacy, Efficiency, and Utility of the Algorithm}

\textbf{Privacy analysis.}  
There are three main components in the privacy analysis. First, we have to define the RKHS to invoke the Gaussian process mechanism in Proposition \ref{clm:hall}. This RKHS should also include the value function approximation we used in the algorithm, namely, neural networks. Second, we give a privacy guarantee of composing the mechanism for $T/B$ iterations. There is not a known composition result of such a functional mechanism, other than the general theorems that apply to any mechanism \cite{kairouz2013composition,dwork2010boosting,beimel2010bounds}. But we derive such a privacy guarantee which is better than existing results. Third, as the sample path is evaluated on multiple different states, the updated value function can be unbounded, which subsequently induces the RKHS norm to be unbounded. This is addressed by showing a probabilistic uniform bound of the sample path over the state space.

Our privacy guarantee is shown in the following theorem.
\begin{theorem}
\label{thm:dpql}
The Q-learning algorithm in Algorithm~\ref{alg:dpql} is $(\epsilon, \delta+J\exp(-(2k-8.68\sqrt\beta\sigma)^2/2))$-DP with respect to two neighboring reward functions $\|r-r^\prime\|_\infty\leq 1$, provided that $2k>8.68\sqrt\beta\sigma$, and
\begin{equation*}
\sigma\geq\sqrt{2(T/B)\ln(e+\epsilon/\delta)}C(\alpha, k, L, B)/\epsilon,
\end{equation*}
where $C(\alpha, k, L, B)=((4\alpha(k+1)/B)^2+4\alpha(k+1)/B)L^2$, $\beta=(4\alpha (k+1)/B)^{-1}$, $L$ is the Lipschitz constant of the value function approximation, $B$ is the batch size, $T$ is the number of iterations, and $\alpha$ is the learning rate.
\end{theorem}

Theorem~\ref{thm:dpql} provides a rigorous guarantee on the privacy of the reward function. We now present three statements to address the challenges mentioned above and support the theorem. 

Lemma \ref{clm:sobolev} and its corollary, informally stated below and formally stated in Appendix \ref{appendix:sobolev}, describe the RKHS that is necessary to both embedded the function approximators we use and invoke the mechanism in Proposition \ref{clm:hall}.

\begin{lemma}[Informal statement]
\label{clm:sobolev}
The Sobolev space $H^1$ with order $1$ and the $\ell^2$-norm is defined as
\begin{equation*}
H^1 = \{f\in C[0,1]: \partial f(x) \text{ exists}; \int_0^1 (\partial f(x))^2 dx < \infty \},
\end{equation*}
where $\partial f(x)$ denotes weak derivatives and the RKHS kernel is $K(x,y)=\exp(-\beta |x-y|)$.
\end{lemma}

Immediately following Lemma \ref{clm:sobolev}, we show that the common neural networks are in the Sobolev space. That includes neural networks with nonlinear activation layers such as a ReLU function, a sigmoid function, or the $\tanh$ function. The proof of the following corollary is also in Appendix \ref{appendix:sobolev}.
\begin{corollary}
\label{nnlipschitz}
Let $\hat{f}_W(x)$ denote the neural network with finitely many finite parameters $W$. For $\hat{f}_W(x)$ with finitely many layers, if the gradient of the activation function is bounded, then $\hat{f}_W(x)\in H^1$.
\end{corollary}

By the corollary $\hat{f}_W(x)$ is Lipschitz continuous. Denote $L$ as the Lipschitz constant which only depends on the network architecture. It follows from Lemma~\ref{clm:sobolev} immediately that, in the algorithm for any $Q(s,a)$ and $Q^\prime(s,a)$, $\|Q(\cdot, a)-Q^\prime(\cdot, a)\|_\HH^2 \leq 2r_0^2(1+\beta/2)/(1-\gamma)^2+L^2/\beta$, for each $a$, where it assumes bounded reward $|r(s,a)|\leq r_0$. This will lead to an alternative privacy guarantee, but less preferred than in Theorem \ref{thm:dpql} due to the $1/(1-\gamma)^2$ and the $r_0$ factor.

Line $19$ and $20$ use $\hat g$, which is the list of Gaussian random variables evaluated at the Gaussian process sample paths. Using a union tail bound we can derive a probabilistic bound of these variables, but it will cause the approximation factor to be $\delta+\OO(1-(1-\exp(2k-\sqrt{\beta}\sigma))^T)$, which is unrealistically large. We show in the lemma below that with high probability the entire sample path is uniformly bounded over any state $s_t$. We can then calibrate the $\delta$ to cover the exponentially small tail bound $\OO(\exp(-u^2))$ of the noise. The proof is in Appendix \ref{appendix:proofpos}.
\begin{lemma}
\label{maxofgp}
Let $\PP$ the probability measure over $H^1$ of the sample path $f$ generated by $\GG(0,\sigma^2 K)$. Then almost surely $\max_{x\in [0,1]} f(x)$ exists, and for any $u>0$
\begin{equation*}
\PP(\max_{x\in [0,1]} f(x) \geq 8.68\sqrt{\beta}\sigma+u)\leq \exp(-u^2/2).
\end{equation*}
\end{lemma}

\begin{proof}[Proof of Theorem~\ref{thm:dpql}]
Let $Q$ and $Q^\prime$ denote the learned value function of the algorithm given $r$ and $r^\prime$, respectively, where $\|r-r^\prime\|_\infty\leq 1$.
To make $Q$ and $Q^\prime$ indistinguishable, we inspect the update step in line $21$. 
Let $Q_0$ denote the value function after and before the update, we have
\[
\|Q-Q_0\|_\infty\leq\alpha L(2+\hat{g}_a(s_{t+1})-\hat{g}_a(s_{t}))/B.
\]
As per Lemma~\ref{maxofgp}, with probability at least $1-\exp(-(2k-8.68\sqrt\beta\sigma)^2/2)$, we have $|Q-Q_0|\leq 2\alpha L(k+1)/B$. By the triangle inequality, for any $\|r-r^\prime\|_\infty\leq 1$, the corresponding $Q$ and $Q^\prime$ satisfies $\|Q-Q^\prime\|_\infty\leq 4\alpha L(k+1)/B$, given that $Q_0$ is fixed by the previous step. Let $f=Q-Q^\prime$, we have 
\[
\|f\|_{\HH}^2\leq (1+\beta/2)(4\alpha L(k+1)/B)^2+L^2/2\beta
\]
by the formal statement Lemma \ref{clm:sobolev}, Appendix \ref{appendix:sobolev}. We choose $1/\beta=4\alpha (k+1)/B$ and have $\|f\|_{\HH}^2\leq ((4\alpha(k+1)/B)^2+4\alpha(k+1)/B)L^2$. Now by Proposition \ref{clm:hall}, adding $g\sim\GG(0,\sigma^2 K)$ to $Q$ will make the update step $(\epsilon^\prime, \delta^\prime+\exp(-(2k-8.68\sqrt\beta\sigma)^2/2)$-differentially private, given that $\sigma\geq \sqrt{2\ln(1.25/\delta^\prime)}\|f\|_{\HH}/\epsilon^\prime$, where $K(x,y)=\exp(-4\alpha L(k+1)|x-y|/B)$ is our choice of the kernel function. Thus each iteration of update has a privacy guarantee. 

It amounts to analyze the composition of $T/B$ many iterations. It is shown by the composition theorem \cite{kairouz2013composition,mironov2017renyi} that any $\sigma\geq \sqrt{2(T/B)\ln(1.25/\delta)\ln(e+\epsilon/\delta)}\|f\|_{\HH}/\epsilon$ is sufficient. This is the best known bound, but we continue to derive the specific bound for our algorithm. Let $z$ (a function, either $Q$ or $Q^\prime$) be the output of a single update of the algorithm. Denote $v=4\alpha(k+1)/B$ and $T^\prime=T/B$ for simplicity. By Lemma~\ref{momentgenerator}, Appendix \ref{appendix:proofpos}, we have 
\begin{align*}
& \EE_{0}[(\PP_1(z\in S)/\PP_0(z\in S))^\lambda] \leq \exp(\frac{(\lambda^2+\lambda)(v^2+v)L^2}{2\sigma^2}),
\end{align*}
where $\PP_0$ and $\PP_1$ are the probability distribution of $z$ given $r$ and $r^\prime$, respectively.
This moment generating function will scale exponentially if multiple independent instances of $z$ are drawn. Namely, let $\bm{z}$ be the vector of $T^\prime$ many independent $z$, and $\PP_0^{T^\prime}$ and $\PP_1^{T^\prime}$ be its probability distribution under $r$ and $r^\prime$. We have for $\lambda>0$, $\EE_{0}[(\PP_1^{T^\prime}(\bm{z}\in S)/\PP_0^{T^\prime}(\bm{z}\in S))^\lambda] \leq \exp(\frac{(\lambda^2+\lambda)(v^2+v)L^2}{2\sigma^2}T^\prime)$.
Thus, 
\[
\exp(\lambda(\ln(\PP_1^{T^\prime}(\bm{z})/\PP_0^{T^\prime}(\bm{z}))-\epsilon)=\exp(\frac{T^\prime\|f\|_{\HH}^2}{2\sigma^2}(\lambda+\frac{1}{2}(1-\frac{2\epsilon\sigma^2}{T^\prime\|f\|_{\HH}^2}))^2-\frac{1}{4}(1-\frac{2\epsilon\sigma^2}{T^\prime\|f\|_{\HH}^2})^2).
\]
Since the argument holds for any $\lambda > 0$, let $\lambda=-\frac{1}{2}(1-\frac{2\epsilon\sigma^2}{T^\prime\|f\|_{\HH}^2})> 0$, then
\begin{align*}
\PP_1^{T^\prime}(\bm{z})-\exp(\epsilon)\PP_0^{T^\prime}(\bm{z})
& \leq \EE_0[\exp(\lambda(\ln(\PP_1^{T^\prime}(\bm{z})/\PP_0^{T^\prime}(\bm{z}))-\epsilon) + \lambda\ln\lambda-(\lambda+1)\ln(\lambda+1))] \\
& = \exp(-\frac{T^\prime\|f\|_{\HH}^2}{2\sigma^2}(1-\frac{2\sigma^2}{T^\prime\|f\|_{\HH}^2}\epsilon)^2 + \lambda\ln\lambda-(\lambda+1)\ln(\lambda+1)) \\
& \leq \exp(-\frac{T^\prime\|f\|_{\HH}^2}{2\sigma^2}(1-\frac{2\sigma^2}{T^\prime\|f\|_{\HH}^2}\epsilon)^2)(\lambda+1) \\
& = \exp(-\frac{\sigma^2}{2T^\prime\|f\|_{\HH}^2}(\epsilon-\frac{T^\prime\|f\|_{\HH}^2}{2\sigma^2})^2)\frac{1}{1+\frac{\sigma^2}{T^\prime\Delta^2}(\epsilon-\frac{T^\prime\|f\|_{\HH}^2}{2\sigma^2})}.
\end{align*}
We desire to find $\epsilon$ and $\delta$ so that this difference $\PP_1^{T^\prime}(\bm{z})-\exp(\epsilon)\PP_0^{T^\prime}(\bm{z})$ is less than $\delta$. We use similar techniques as is in the proof Theorem 4.3 of \cite{kairouz2013composition}. We choose 
\[
\epsilon=\frac{T^\prime\|f\|_{\HH}^2}{2\sigma^2}+\sqrt{\frac{2T^\prime\|f\|_{\HH}^2w}{\sigma^2}},
\]
where $w=\ln(e+\sqrt{T^\prime\|f\|_{\HH}^2/\sigma^2}/\delta)$. Thus the first term is $e^{-w}$ and the second term is $\frac{1}{1+\sqrt{\frac{2\sigma^2w}{T^\prime \|f\|_{\HH}^2}}}$. 
This ensures that $e^{-w}\leq \delta/\sqrt{T^\prime\|f\|_{\HH}^2/\sigma^2}$ and $\frac{1}{1+\sqrt{\frac{2\sigma^2w}{T^\prime\|f\|_{\HH}^2}}}\leq\frac{1}{1+\sqrt{\frac{2\sigma^2}{T^\prime\|f\|_{\HH}^2}}}$, thereby guaranteeing that
%. This guarantees that 
$\PP_1^{T^\prime}(\bm{z})-\exp(\epsilon)\PP_0^{T^\prime}(\bm{z}) \leq \delta$ for differential privacy. 
We solve $\epsilon=\frac{T^\prime\|f\|_{\HH}^2}{2\sigma^2}+\sqrt{\frac{2T^\prime\|f\|_{\HH}^2w}{\sigma^2}}$ and find the sufficient condition that
\begin{align*}
\sigma & = \sqrt{(2T^\prime\|f\|_{\HH}^2/\epsilon^2)\ln(e+\epsilon/\delta)} \\
& \leq \sqrt{2(T/B)((4\alpha(k+1)/B)^2+4\alpha(k+1)/B)L^2\ln(e+\epsilon/\delta)}/\epsilon,
\end{align*}
as desired. When this sufficient condition is satisfied, the approximation factor will be no larger than $\delta$ plus $J$ times the uniform bound derived above by Lemma \ref{maxofgp}. Namely, it achieves $(\epsilon, \delta+J\exp(-(2k-8.68\sqrt\beta\sigma)^2/2))$-DP.
\end{proof}

\textbf{Time complexity.}  We show that the noise adding mechanism in our algorithm is efficient. In fact, the most complex step introduced by the noise-adding is the insertion in line $15$. This can be negligible compared with the steps such as computing gradients and executing actions. A complete proof of the below proposition is given in Appendix \ref{appendix:proofpos}. 
\begin{proposition}
\label{efficiency}
The noised value function (during either training or released) in Algorithm~\ref{alg:dpql} can respond to $N_q$ queries in $\OO(N_q\ln(N_q))$ time.
\end{proposition}

\textbf{Utility analysis.} To the best of our knowledge, there has not been a study to rigorously analyze the utility of deep reinforcement learning. In fact, in the continuous state space setting, the solution of the Bellman equation is not unique in general. Hence, it is unlikely for Q-learning to achieve a guaranteed performance, even if it converges. However, we gain insights by analyzing the algorithm's learning error in the discrete state space setting. The learning error is defined by the discrepancy between the learned state value function and the ground truth of the optimal state value function. We consider the worst case $J=1$ for utility (which is the best case for $(\epsilon, \delta+\exp(-(2k-8.68\sqrt\beta\sigma)^2/2))$-differential privacy) where the noise is the most correlated through the iterations. We show the upper bound of the utility loss, which has a limit of zero as the number of states approaches infinity. The proof involves the linear program formulation of MDP, which is given in Appendix \ref{sec:utility}.

\begin{proposition}
\label{utility}
Let $v^\prime$ and $v^\ast$ be the value function learned by our algorithm and the optimal value function, respectively. In the case $J=1$, $|S|=n<\infty$, and $\gamma<1$, the utility loss of the algorithm satisfies
\begin{equation*}
\EE[\frac{1}{n}\|v^\prime-v^\ast\|_1]\leq \frac{2\sqrt{2}\sigma}{\sqrt{n\pi}(1-\gamma)}.
\end{equation*}
\end{proposition}

\subsection{Discussion}

\textbf{Extending to other RL algorithms.} Our algorithm can be extended to the case where it learns both a policy function and a value function coordinately, like the actor-critic method \cite{mnih2016asynchronous}. If in the updates of the policy function, only the $Q$ function is used in the policy gradient estimation, for example $\nabla_{\theta_\pi}\ln\pi(a|s)Q(s,a)$, then the algorithm has the same privacy guarantee. Also, any post-processing of the private $Q$ function will not break the privacy guarantee. This also includes experience replay and $\epsilon$-greedy policies~\cite{mnih2015human}.

However, consider the case where the reward is directly accessed in the policy gradient estimation, for example $\nabla_{\theta_\pi}\ln\pi(a|s)A(s,a)$ in \cite{degris2012off,schulman2015high} where 
$A(s,a)=\sum_{t=0}^T(\lambda\gamma)^t(r_t+v(s_{t+1})-v(s_t))$. In this setting, the privacy guarantee no longer holds. To extend the privacy guarantee to this case, one should add noise to the policy function as well.

\textbf{Extending to high-dimensional tasks.} We assumed in our analysis $\cS=[0,1]$ for simplicity. The setting extends to any bounded $\cS\subseteq \RR$ by scaling the interval.
Our approach can also be extended to high-dimensional spaces by choosing a high-dimensional RKHS and kernel. For example, the kernel function $\exp(-\beta|x-y|)$ where $x$ and $y$ now belongs to $R^n$ and $|\cdot|$ is the Manhattan distance. It is also possible to use other RKHS and kernels for the Gaussian process noise, such as the space of band-limited functions. Other than the re-calibration of the noise level to the new kernel, the privacy guarantee in the theorem holds in general for the respective definition of $\|\cdot\|_\HH$. 
We note that the time complexity derived in Proposition \ref{efficiency} does not extend to other kernel functions, which requires the algorithm to take $\OO(N_q^2)$ in the noise generating process.

\section{Experiments}
\label{sec:numresults}

We present empirical results to corroborate our theoretical guarantees and to demonstrate the performance of the proposed algorithm on a small example. The exact MDP we use is described in Appendix \ref{appendix:experiments-env}. The implementation is attached along with the manuscript submission. 

We first plot the learning curve with a variety of noise levels and $J$ values in Figure \ref{fig:empirical-noise}. 
With the increase of the noise level, the algorithm requires more samples to achieve the same return than the non-private version.
This demonstrates the empirical privacy-utility tradeoff.
We observe that with the noise being reset every round ($J=T/B$), the algorithm is likely to converge with limited sub-optimality as desired, especially when $\sigma<0.4$. 
Therefore, as $J\exp(-(2k-8.68\sqrt\beta\sigma)^2/2)$ is exponentially small, we suggest using $J=T/B$ in practice to achieve a better utility. 

The algorithm is then compared with a variety of baseline methods where they target the same $(\epsilon, \delta)$ privacy guarantee, as shown in Figure \ref{fig:empirical-compare}(a) and \ref{fig:empirical-compare}(b). The policy evaluation method proposed by Balle, Gomrokchi and Precup \cite{balle2016differentially} is not differentially private under our context (while it is $(\epsilon, \delta)$-DP with respect to the reward sequences). We compare with it to illustrative the utility, where it is observed that their approach shares similar performance with ours. Note that studies on contextual bandits by Sajed and Sheffet \cite{sajed2019optimal} and by Shariff and Sheffet \cite{shariff2018differentially} consider an equivalently one-step MDP as \cite{balle2016differentially} and thus will yield the same method. We also compared our approach with the input perturbation method proposed by Venkitasubramaniam~\cite{venkitasubramaniam2013privacy} and the differentially private deep learning framework by Abadi et al. \cite{abadi2016deep}. Both the approaches are differentially private under our setting, while our algorithm significantly outperforms them. Especially, on the higher privacy regime $\epsilon=0.45$, both the baseline methods do not improve over the training due the the large noise level needed. The baseline implementations and the exact calculation of the parameters are detailed in Appendix \ref{appendix:experiments-baselines} and \ref{appendix:experiments-parameters}, respectively.

\begin{figure}[t!]
\vskip -0.1in
\begin{center}
\includegraphics[width=0.48\columnwidth]{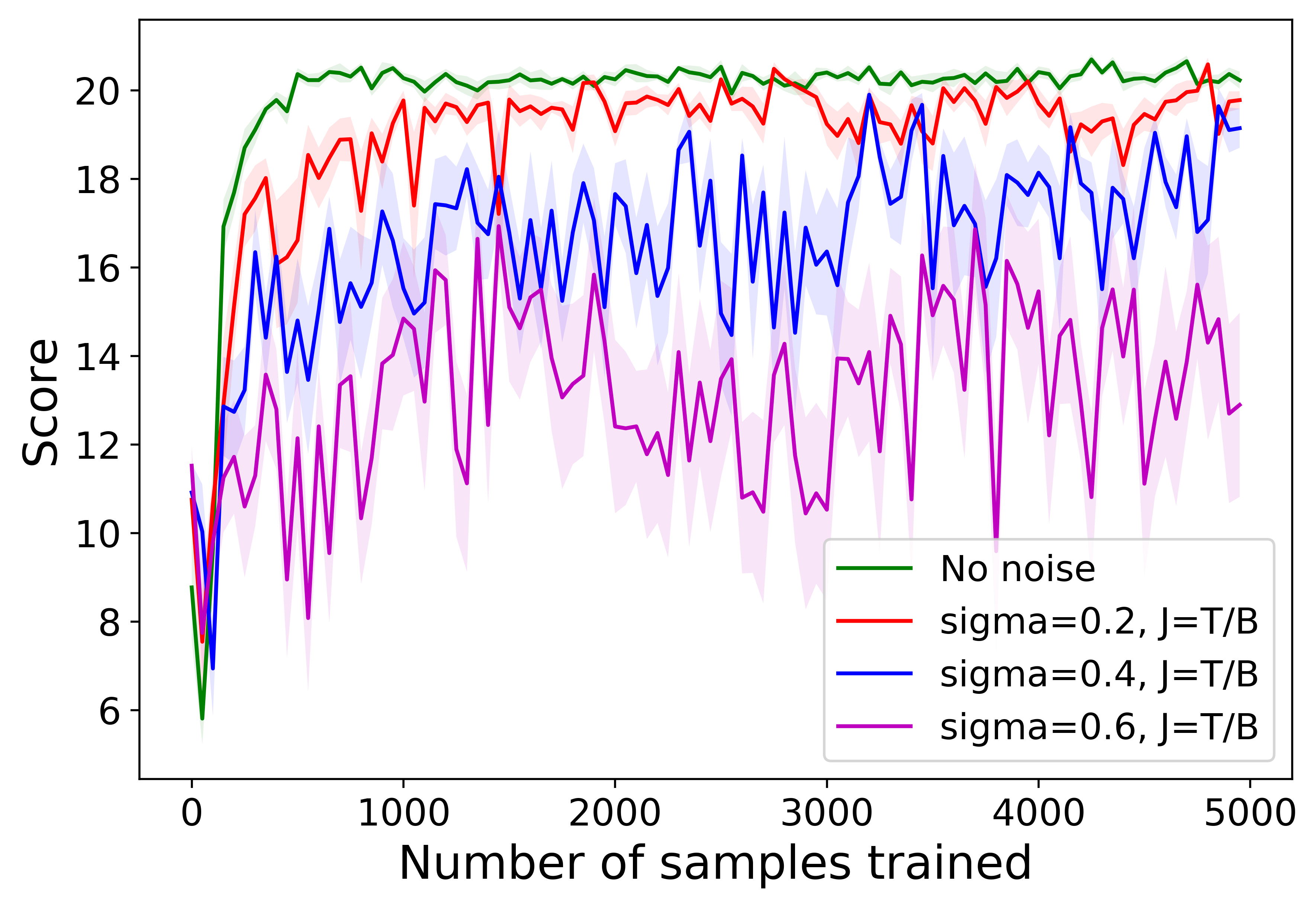}
\includegraphics[width=0.48\columnwidth]{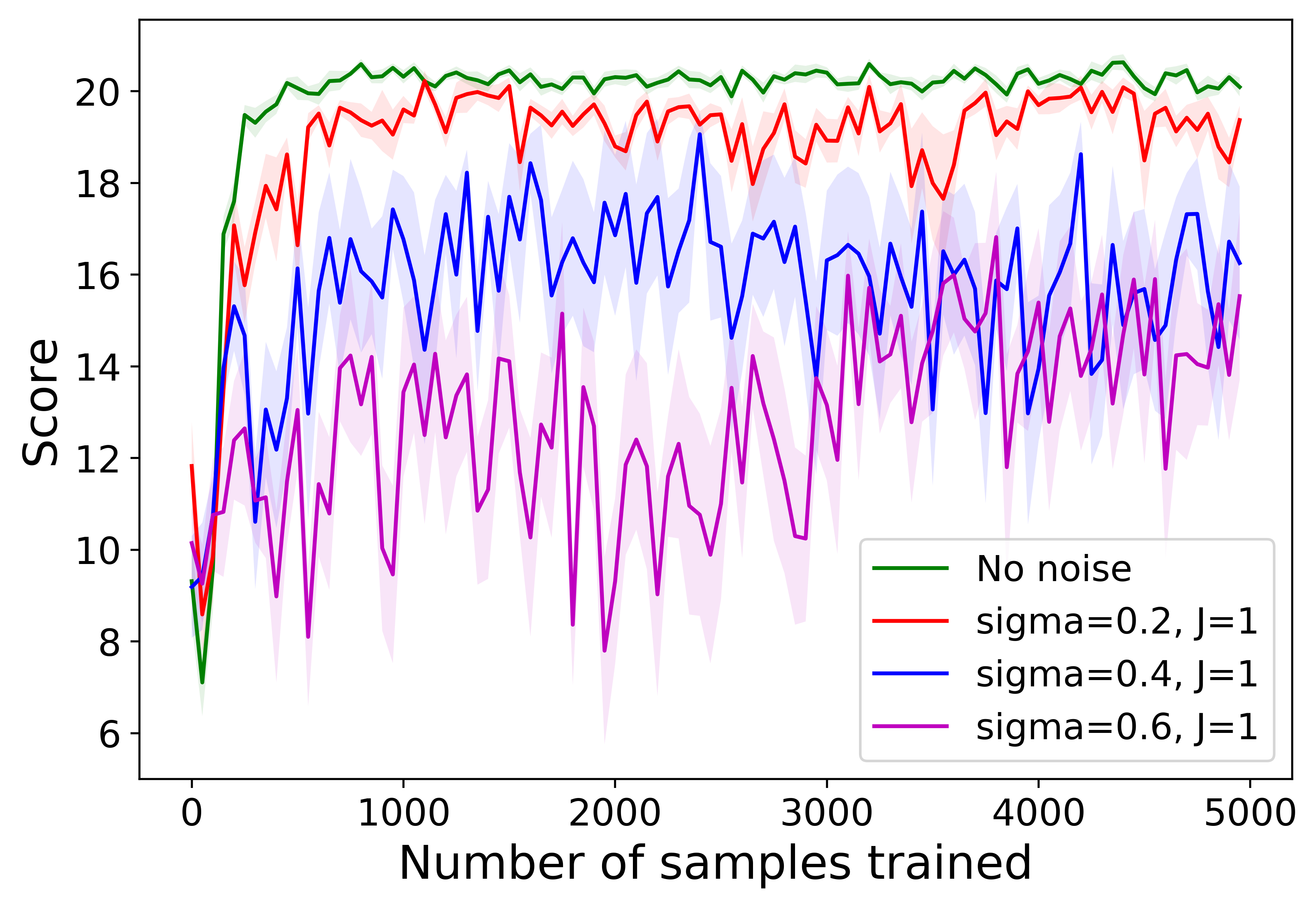}
\caption{Empirical results of Algorithm \ref{alg:dpql} on different noise levels. The y-axis denotes the return. The x-axis is the number of samples the agent has trained on. Each episode has 50 samples. The shadow denotes 1-std. The learning curves are averaged over 10 random seeds. The curves are generated without smoothing.}
\label{fig:empirical-noise}
\end{center}
\vskip -0.05in
\end{figure}

\begin{figure}[t!]
\vskip -0.05in
\begin{center}
\subfloat[Target $\epsilon=0.9$, $\delta=1\cdot 10^{-4}$]{\includegraphics[width=0.48\columnwidth]{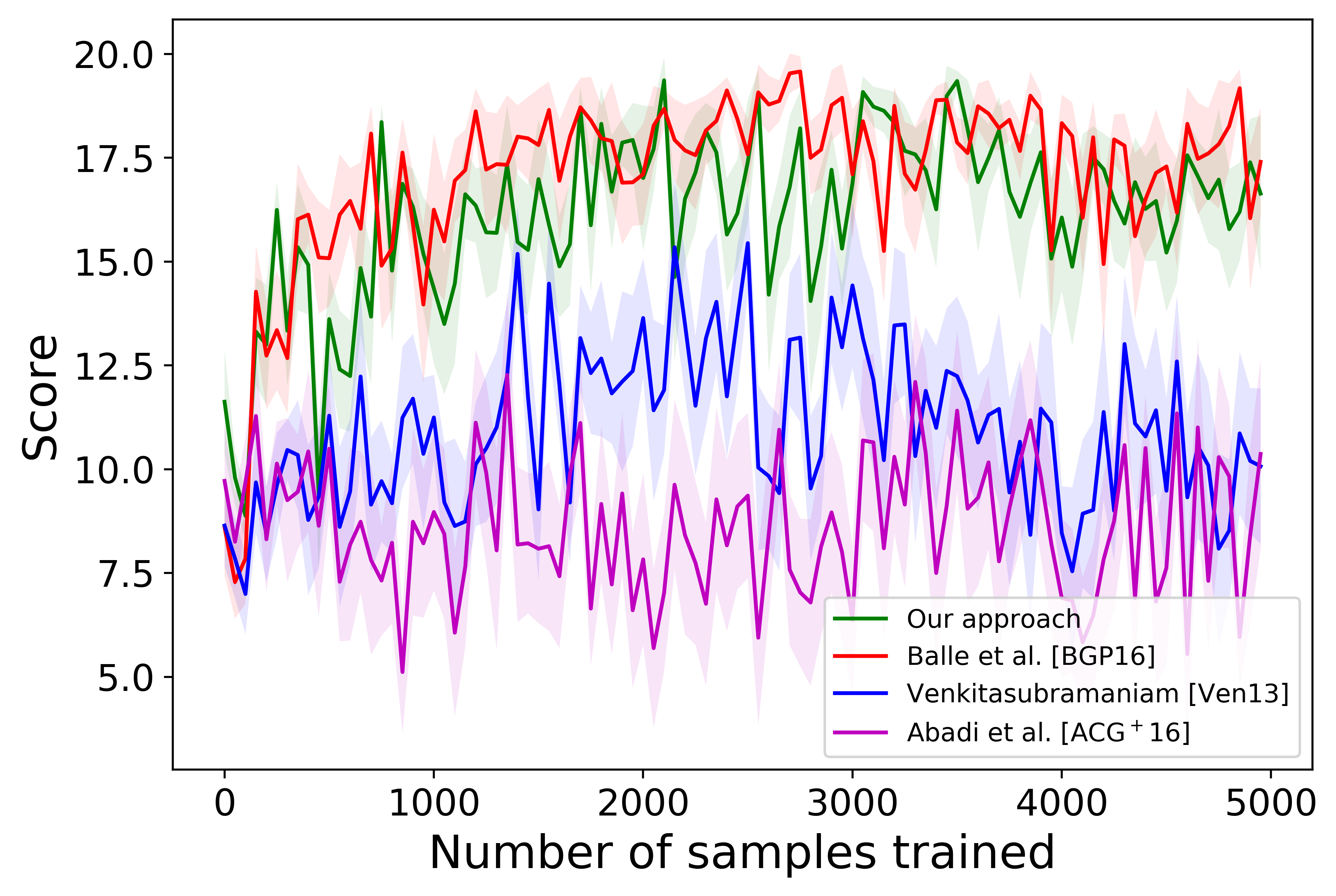}}
\subfloat[Target $\epsilon=0.45$, $\delta=1\cdot 10^{-4}$]{\includegraphics[width=0.48\columnwidth]{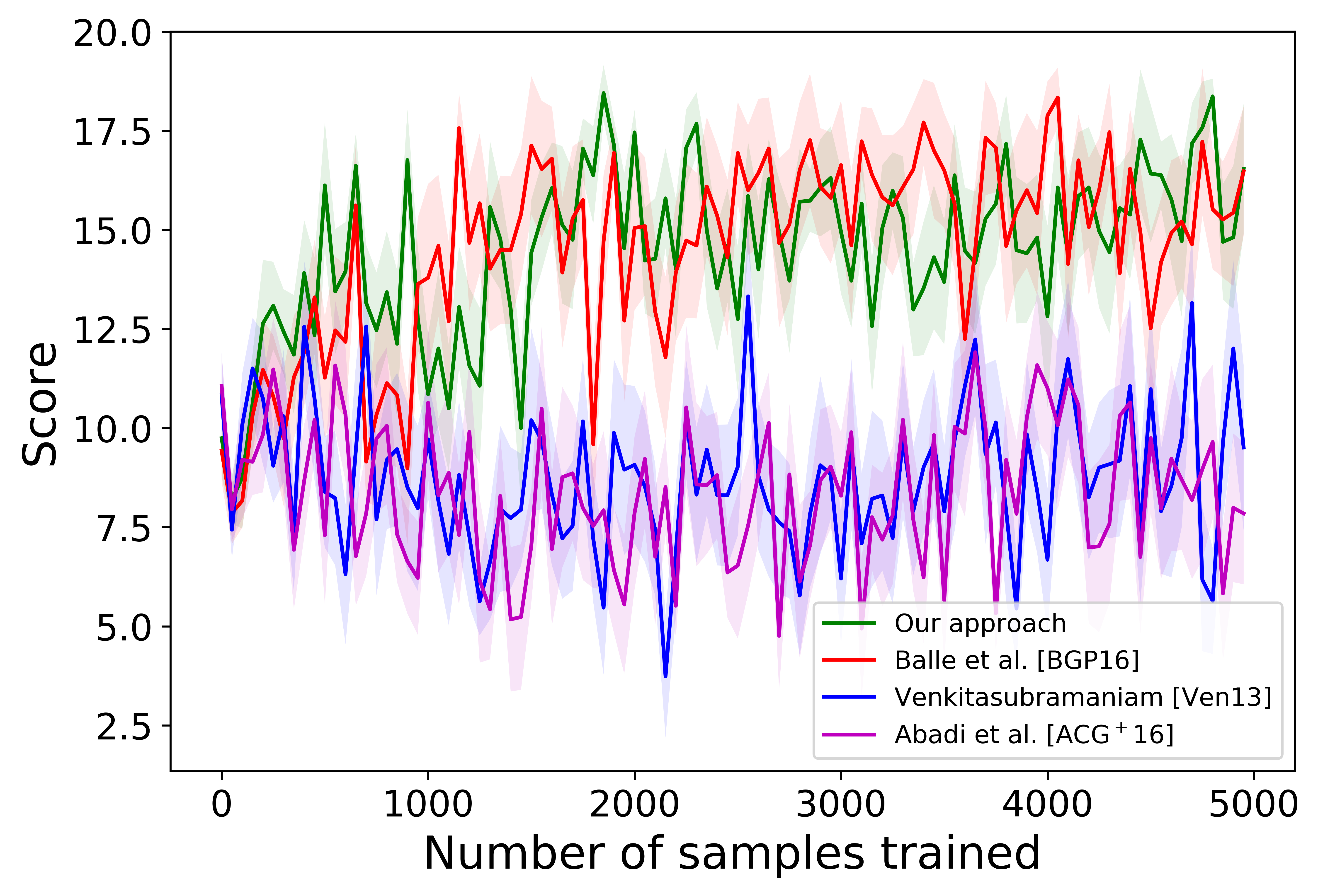}}
\caption{Empirical comparisons with other methods. Same configurations as Figure \ref{fig:empirical-noise}. }
\label{fig:empirical-compare}
\end{center}
\vskip -0.2in
\end{figure}

\section{Conclusion}
We have developed a rigorous and efficient algorithm for differentially private Q-learning in continuous state space settings. Releasing and querying the algorithm's output value function will not distinguish two neighboring reward functions. To achieve this, our method applies functional noise taken from sample paths of a Gaussian process calibrated appropriately according to sensitivity calculated under the RKHS measure. Theoretically, we show the privacy guarantee and insights into the utility analysis. Empirically, experiments corroborate our theoretical findings and show improvement over existing methods. Our approach is general enough to be extended to other domains beyond reinforcement learning.

\section*{Acknowledgement}

We would like to thank Ruitong Huang, who provides helpful insight on the composition analysis and the algorithm design, and Kry Yik Chau Lui, who points out the idea to extend our approach to high-dimensional Sobolev spaces. 

\bibliographystyle{alpha}
\bibliography{neurips_2019}

\newcommand{\etalchar}[1]{$^{#1}$}
\begin{thebibliography}{KGGW15}

\bibitem[ACG{\etalchar{+}}16]{abadi2016deep}
Martin Abadi, Andy Chu, Ian Goodfellow, H~Brendan McMahan, Ilya Mironov, Kunal
  Talwar, and Li~Zhang.
\newblock Deep learning with differential privacy.
\newblock In {\em Proceedings of the 2016 ACM SIGSAC Conference on Computer and
  Communications Security}, pages 308--318. ACM, 2016.

\bibitem[ALMT17]{abernethy2017online}
Jacob Abernethy, Chansoo Lee, Audra McMillan, and Ambuj Tewari.
\newblock Online linear optimization through the differential privacy lens.
\newblock {\em arXiv preprint arXiv:1711.10019}, 2017.

\bibitem[AN04]{abbeel2004apprenticeship}
Pieter Abbeel and Andrew~Y Ng.
\newblock Apprenticeship learning via inverse reinforcement learning.
\newblock In {\em International Conference on Machine Learning}, page~1. ACM,
  2004.

\bibitem[AS17]{agarwal2017price}
Naman Agarwal and Karan Singh.
\newblock The price of differential privacy for online learning.
\newblock In {\em International Conference on Machine Learning}, pages 32--40,
  2017.

\bibitem[Bai95]{baird1995residual}
Leemon Baird.
\newblock Residual algorithms: Reinforcement learning with function
  approximation.
\newblock In {\em Machine Learning Proceedings 1995}, pages 30--37. Elsevier,
  1995.

\bibitem[BGP16]{balle2016differentially}
Borja Balle, Maziar Gomrokchi, and Doina Precup.
\newblock Differentially private policy evaluation.
\newblock In {\em International Conference on Machine Learning}, pages
  2130--2138, 2016.

\bibitem[BKN10]{beimel2010bounds}
Amos Beimel, Shiva~Prasad Kasiviswanathan, and Kobbi Nissim.
\newblock Bounds on the sample complexity for private learning and private data
  release.
\newblock In {\em Theory of Cryptography Conference}, pages 437--454. Springer,
  2010.

\bibitem[CBK{\etalchar{+}}19]{chamikara2019local}
MAP Chamikara, P~Bertok, I~Khalil, D~Liu, and S~Camtepe.
\newblock Local differential privacy for deep learning.
\newblock {\em arXiv preprint arXiv:1908.02997}, 2019.

\bibitem[CW16]{chen2016stochastic}
Yichen Chen and Mengdi Wang.
\newblock Stochastic primal-dual methods and sample complexity of reinforcement
  learning.
\newblock {\em arXiv preprint arXiv:1612.02516}, 2016.

\bibitem[DFVR02]{de2002linear}
Daniela~Pucci De~Farias and Benjamin Van~Roy.
\newblock {\em The linear programming approach to approximate dynamic
  programming: Theory and application}.
\newblock PhD thesis, Stanford University, 2002.

\bibitem[DKM{\etalchar{+}}06]{dwork2006our}
Cynthia Dwork, Krishnaram Kenthapadi, Frank McSherry, Ilya Mironov, and Moni
  Naor.
\newblock Our data, ourselves: Privacy via distributed noise generation.
\newblock In {\em Annual International Conference on the Theory and
  Applications of Cryptographic Techniques}, pages 486--503. Springer, 2006.

\bibitem[DMNS06]{dwork2006calibrating}
Cynthia Dwork, Frank McSherry, Kobbi Nissim, and Adam Smith.
\newblock Calibrating noise to sensitivity in private data analysis.
\newblock In {\em Theory of cryptography conference}, pages 265--284. Springer,
  2006.

\bibitem[DR14]{dwork2014algorithmic}
Cynthia Dwork and Aaron Roth.
\newblock The algorithmic foundations of differential privacy.
\newblock {\em Foundations and Trends{\textregistered} in Theoretical Computer
  Science}, 9(3--4):211--407, 2014.

\bibitem[DRV10]{dwork2010boosting}
Cynthia Dwork, Guy~N Rothblum, and Salil Vadhan.
\newblock Boosting and differential privacy.
\newblock In {\em 2010 IEEE 51st Annual Symposium on Foundations of Computer
  Science}, pages 51--60. IEEE, 2010.

\bibitem[DWS12]{degris2012off}
Thomas Degris, Martha White, and Richard~S Sutton.
\newblock Off-policy actor-critic.
\newblock {\em arXiv preprint arXiv:1205.4839}, 2012.

\bibitem[GUK17]{gajane2017corrupt}
Pratik Gajane, Tanguy Urvoy, and Emilie Kaufmann.
\newblock Corrupt bandits for preserving local privacy.
\newblock {\em arXiv preprint arXiv:1708.05033}, 2017.

\bibitem[HDZ{\etalchar{+}}18]{hu2018reinforcement}
Yujing Hu, Qing Da, Anxiang Zeng, Yang Yu, and Yinghui Xu.
\newblock Reinforcement learning to rank in e-commerce search engine:
  Formalization, analysis, and application.
\newblock In {\em Proceedings of the 24th ACM SIGKDD International Conference
  on Knowledge Discovery \& Data Mining}, pages 368--377. ACM, 2018.

\bibitem[HRRU14]{hsu2014privately}
Justin Hsu, Aaron Roth, Tim Roughgarden, and Jonathan Ullman.
\newblock Privately solving linear programs.
\newblock In {\em International Colloquium on Automata, Languages, and
  Programming}, pages 612--624. Springer, 2014.

\bibitem[HRW13]{hall2013differential}
Rob Hall, Alessandro Rinaldo, and Larry Wasserman.
\newblock Differential privacy for functions and functional data.
\newblock {\em Journal of Machine Learning Research}, 14(Feb):703--727, 2013.

\bibitem[JKT12]{jain2012differentially}
Prateek Jain, Pravesh Kothari, and Abhradeep Thakurta.
\newblock Differentially private online learning.
\newblock In {\em Conference on Learning Theory}, pages 24--1, 2012.

\bibitem[KGGW15]{kusner2015differentially}
Matt Kusner, Jacob Gardner, Roman Garnett, and Kilian Weinberger.
\newblock Differentially private bayesian optimization.
\newblock In {\em International Conference on Machine Learning}, pages
  918--927, 2015.

\bibitem[KLN{\etalchar{+}}11]{kasiviswanathan2011can}
Shiva~Prasad Kasiviswanathan, Homin~K Lee, Kobbi Nissim, Sofya Raskhodnikova,
  and Adam Smith.
\newblock What can we learn privately?
\newblock {\em SIAM Journal on Computing}, 40(3):793--826, 2011.

\bibitem[KOV13]{kairouz2013composition}
Peter Kairouz, Sewoong Oh, and Pramod Viswanath.
\newblock The composition theorem for differential privacy.
\newblock {\em arXiv preprint arXiv:1311.0776}, 2013.

\bibitem[Lal13]{lalley2018gp}
Steven~P. Lalley.
\newblock Introduction to gaussian processes.
\newblock
  https://galton.uchicago.edu/~lalley/Courses/386/GaussianProcesses.pdf, 2013.
\newblock Accessed: 2018-12-27.

\bibitem[LSTS15]{liebman2015dj}
Elad Liebman, Maytal Saar-Tsechansky, and Peter Stone.
\newblock Dj-mc: A reinforcement-learning agent for music playlist
  recommendation.
\newblock In {\em Proceedings of the 2015 International Conference on
  Autonomous Agents and Multiagent Systems}, pages 591--599, 2015.

\bibitem[MBM{\etalchar{+}}16]{mnih2016asynchronous}
Volodymyr Mnih, Adria~Puigdomenech Badia, Mehdi Mirza, Alex Graves, Timothy
  Lillicrap, Tim Harley, David Silver, and Koray Kavukcuoglu.
\newblock Asynchronous methods for deep reinforcement learning.
\newblock In {\em International Conference on Machine Learning}, pages
  1928--1937, 2016.

\bibitem[Mir17]{mironov2017renyi}
Ilya Mironov.
\newblock Renyi differential privacy.
\newblock In {\em Computer Security Foundations Symposium (CSF), 2017 IEEE
  30th}, pages 263--275. IEEE, 2017.

\bibitem[MKS{\etalchar{+}}15]{mnih2015human}
Volodymyr Mnih, Koray Kavukcuoglu, David Silver, Andrei~A Rusu, Joel Veness,
  Marc~G Bellemare, Alex Graves, Martin Riedmiller, Andreas~K Fidjeland, Georg
  Ostrovski, et~al.
\newblock Human-level control through deep reinforcement learning.
\newblock {\em Nature}, 518(7540):529, 2015.

\bibitem[MT15]{mishra2015nearly}
Nikita Mishra and Abhradeep Thakurta.
\newblock (nearly) optimal differentially private stochastic multi-arm bandits.
\newblock In {\em Proceedings of the Thirty-First Conference on Uncertainty in
  Artificial Intelligence}, pages 592--601. AUAI Press, 2015.

\bibitem[NR00]{ng2000algorithms}
Andrew~Y Ng and Stuart~J Russell.
\newblock Algorithms for inverse reinforcement learning.
\newblock In {\em International Conference on Machine Learning}, pages
  663--670. Morgan Kaufmann Publishers Inc., 2000.

\bibitem[PFW{\etalchar{+}}18]{pan2018reinforcement}
Yangchen Pan, Amir-massoud Farahmand, Martha White, Saleh Nabi, Piyush Grover,
  and Daniel Nikovski.
\newblock Reinforcement learning with function-valued action spaces for partial
  differential equation control.
\newblock {\em arXiv preprint arXiv:1806.06931}, 2018.

\bibitem[RJG{\etalchar{+}}18]{rosset2018optimizing}
Corby Rosset, Damien Jose, Gargi Ghosh, Bhaskar Mitra, and Saurabh Tiwary.
\newblock Optimizing query evaluations using reinforcement learning for web
  search.
\newblock In {\em The 41st International ACM SIGIR Conference on Research \&
  Development in Information Retrieval}, pages 1193--1196. ACM, 2018.

\bibitem[SB18]{sutton2018reinforcement}
Richard~S Sutton and Andrew~G Barto.
\newblock {\em Reinforcement learning: An introduction}.
\newblock MIT press, 2018.

\bibitem[SML{\etalchar{+}}15]{schulman2015high}
John Schulman, Philipp Moritz, Sergey Levine, Michael Jordan, and Pieter
  Abbeel.
\newblock High-dimensional continuous control using generalized advantage
  estimation.
\newblock {\em arXiv preprint arXiv:1506.02438}, 2015.

\bibitem[SS18]{shariff2018differentially}
Roshan Shariff and Or~Sheffet.
\newblock Differentially private contextual linear bandits.
\newblock In {\em Advances in Neural Information Processing Systems}, pages
  4301--4311, 2018.

\bibitem[SS19]{sajed2019optimal}
Touqir Sajed and Or~Sheffet.
\newblock An optimal private stochastic-mab algorithm based on an optimal
  private stopping rule.
\newblock In {\em International Conference on Machine Learning}, 2019.

\bibitem[Sze10]{szepesvari2010algorithms}
Csaba Szepesv{\'a}ri.
\newblock Algorithms for reinforcement learning.
\newblock {\em Synthesis lectures on artificial intelligence and machine
  learning}, 4(1):1--103, 2010.

\bibitem[TD16]{tossou2016algorithms}
Aristide~CY Tossou and Christos Dimitrakakis.
\newblock Algorithms for differentially private multi-armed bandits.
\newblock In {\em Thirtieth AAAI Conference on Artificial Intelligence}, 2016.

\bibitem[TD17]{tossou2017achieving}
Aristide Charles~Yedia Tossou and Christos Dimitrakakis.
\newblock Achieving privacy in the adversarial multi-armed bandit.
\newblock In {\em Thirty-First AAAI Conference on Artificial Intelligence},
  2017.

\bibitem[TS13]{thakurta2013nearly}
Abhradeep~Guha Thakurta and Adam Smith.
\newblock (nearly) optimal algorithms for private online learning in
  full-information and bandit settings.
\newblock In {\em Advances in Neural Information Processing Systems}, pages
  2733--2741, 2013.

\bibitem[Ven13]{venkitasubramaniam2013privacy}
Parv Venkitasubramaniam.
\newblock Privacy in stochastic control: A markov decision process perspective.
\newblock In {\em Communication, Control, and Computing (Allerton), 2013 51st
  Annual Allerton Conference on}, pages 381--388. IEEE, 2013.

\bibitem[WD92]{watkins1992q}
Christopher~JCH Watkins and Peter Dayan.
\newblock Q-learning.
\newblock {\em Machine learning}, 8(3-4):279--292, 1992.

\bibitem[ZZZ{\etalchar{+}}18]{zheng2018drn}
Guanjie Zheng, Fuzheng Zhang, Zihan Zheng, Yang Xiang, Nicholas~Jing Yuan, Xing
  Xie, and Zhenhui Li.
\newblock Drn: A deep reinforcement learning framework for news recommendation.
\newblock In {\em Proceedings of the 2018 World Wide Web Conference on World
  Wide Web}, pages 167--176. International World Wide Web Conferences Steering
  Committee, 2018.

\end{thebibliography}

\newpage
\allowdisplaybreaks
\appendix

\section{Proof of Lemma \ref{maxofgp}, Lemma \ref{momentgenerator} and Proposition \ref{efficiency}}
\label{appendix:proofpos}

In the proofs, we will refer to some properties of the Gaussian process and its sample paths. We put those properties in Claim \ref{clm:basic}, at the end of this section. The claim and the notation used in the claim are reused in the proof of Lemma \ref{maxofgp}, Lemma \ref{momentgenerator}, and Proposition \ref{efficiency}. For simplicity, when the result can be verified immediately, such as calculating an inverse matrix, we will omit the steps.

\textbf{Notations for this section.} We investigate a dyadic rational. For a sample path $f$ and $n\geq 1$, define $f_{n0}=\{f(x_0), f(x_2),\dots, f(x_{2n})\}$ and $f_{n1}=\{f(x_1), f(x_3),\dots f(x_{2n-1})\}$, where $x_{i}=i/2n$, $i=0,\dots, 2n$. For a deterministic function $g$ defined on $[0,1]$, define $g_{n0}$ and $g_{n1}$ similarly as $g_{n0}=(\lim_{x\rightarrow x_0} g(x), \lim_{x\rightarrow x_2} g(x),\dots, \lim_{x\rightarrow x_{2n}} g(x))^T$ and $g_{n1}=(\lim_{x\rightarrow x_1} g(x), \lim_{x\rightarrow x_3} g(x),\dots \lim_{x\rightarrow x_{2n-1}} g(x))^T$. Also define $\beta_n=\beta/2n$. Our goal is to investigate the desired properties when $\lim_{n\rightarrow \infty}$.

By the definition of the Gaussian process, 
\begin{equation*}
\begin{pmatrix}
f_{n1}\\
f_{n0}
\end{pmatrix}
\sim \NN \left(0, \sigma^2
\begin{bmatrix}
K_{11} & K_{10} \\
K_{10}^T & K_{00}
\end{bmatrix}
\right), 
\end{equation*}
where $K_{11}$, $K_{10}$, and $K_{00}$ are depending on $n$. If $f\sim \GG(0, \sigma^2K)$, the conditional distribution $f_{n1}|f_{n0} \sim \NN(K_{10}K_{00}^{-1}f_{n0}, \sigma^2(K_{11}-K_{10}K_{00}^{-1}K_{10}^{T}))$. If $f\sim \GG(g, \sigma^2K)$, $f_{n1}|f_{n0} \sim \NN(g_{n1}+K_{10}K_{00}^{-1}(f_{n0}-g_{n0}), \sigma^2(K_{11}-K_{10}K_{00}^{-1}K_{10}^{T}))$.

We restate Lemma \ref{maxofgp} and prove it. Recall that the lemma proves a tail bound for the maximum of a GP sample path. 

\begin{namedtheorem}[Lemma \ref{maxofgp}]
Let $\PP$ the probability measure over $H^1$ of the sample path $f$ generated by $\GG(0,\sigma^2 K)$. Then almost surely $\max_{x\in [0,1]} f(x)$ exists, and for any $u>0$
\begin{equation*}
\PP(\max_{x\in [0,1]} f(x) \geq 8.68\sqrt{\beta}\sigma+u)\leq \exp(-u^2/2).
\end{equation*}
\end{namedtheorem}

\begin{proof}
We start with the base case that $f_{10}=\{f(0), f(1)\}$. $\EE[\max f_{10}]= \frac{1}{2}\EE[|f(0)-f(1)|]= \sqrt{(1-\exp(-\beta))/\pi}\sigma\leq \sqrt{\beta/\pi}\sigma$, where the second equation is due to $f(0)-f(1)\sim \NN(0, 2(1-\exp(-\beta))\sigma^2)$.

We desire to bound the expectation $\EE[\max f_{n0}]$ for all $n>1$. To complete this, we first prove a general result on the expectation of the maximum of the Gaussian variables. Let $z_1, \dots, z_n \sim \NN(0, \sigma_z^2)$ be $n$ independent Gaussian random variables, by the Chernoff bound we have
\begin{align*}
\exp(t\EE [\max_i z_i]) \leq \EE[\exp(t\max_i z_i)] = \EE[\max_i\exp(tz_i)] \leq n\exp(t^2\sigma_z^2/2).
\end{align*}
By choosing $t=\sqrt{2\ln n}/\sigma_z$, we conclude that $\EE [\max_i z_i]\leq \sqrt{2\ln n}\sigma_z$. Meanwhile, it is obvious that $\Var[\max_i z_i]\leq \sigma_z$.

Denote $\mu_n = \EE [\max f_{2^n0}]$. As $f_{2^n0}\subset f_{2^{n+1}0}$, the series $\mu_n$ is non-decreasing. We derive an upperbound of $\mu_{n+1}-\mu_n$. Let $x_i=i/2^{n+1}$ and $\xi_{i,n}=f(x_{2i-1})-\frac{\exp(-\beta_{2^n})}{1+\exp(-2\beta_{2^n})}(f(x_{2i-2})+f(x_{2i}))$, we have that $\xi_{i,n} \sim \NN(0, \sigma^2\frac{1-\exp(-2\beta_{2^n})}{1+\exp(-2\beta_{2^n})})$. Further, $\xi_{i,n}$ and $\xi_{j,n}$ are independent. It is true as $K_{10}K_{00}^{-1}$ is nonzero only at its two diagonals, which indicates that $f(x_{2j-1})$ is not depending on other point if $f(x_{2j-2})+f(x_{2j})$ is given. Thus, we have $\xi_{i,n}$ i.i.d. for $i=1,\dots,2^n$.

As we shown before the upper bound $\EE [\max_i z_i]\leq 2\sqrt{\ln n}\sigma_z$, the expectation is monotonically increasing on the variance $\sigma_z$. In general, we have $(1-\exp(cx))/(1+\exp(cx))<x$ for all positive $x$, if and only if $c\leq 2$. Thus $\frac{1-\exp(-2\beta_{2^n})}{1+\exp(-2\beta_{2^n})}<\beta_{2^n}$, and consequently we have $\EE [\max_i \xi_{i,n}]\leq \EE [\max_i \xi_{i,n}^\prime]$, for $\xi_{i,n}^\prime \sim \NN(0, \sigma^2\beta_{2^n})$.

By the inequality $\exp(-x)/(1+\exp(-2x))<\frac{1}{2}$ for $x>0$, we relax that $\frac{\exp(-\beta_n)}{1+\exp(-2\beta_n)}<\frac{1}{2}$. Thus, $f(x_{2i-1})\leq \xi_{i,n}+\frac{1}{2}(f(x_{2i-2})+f(x_{2i}))$. Taking maximum and expectation on both sides, we have 
\begin{align*}
\EE[\max_{1\leq i \leq 2^n} f(x_{2i-1})] & \leq \EE[\max_{1\leq i \leq 2^n}\xi_{i,n}+\frac{1}{2}(f(x_{2i-2})+f(x_{2i}))] \\
& \leq \EE[\max_{1\leq i \leq 2^n}\xi_{i,n}^\prime]+\frac{1}{2}\EE[\max_{1\leq i \leq 2^n}f(x_{2i-2})]+\frac{1}{2}\EE[\max_{1\leq i \leq 2^n}f(x_{2i})] \\
& \leq \sqrt{2\ln 2^n}\sqrt{\sigma^2\beta_{2^n}} + \frac{1}{2}\mu_n + \frac{1}{2}\mu_n \\
& = \mu_n + \sqrt{n\beta/2^n}\sigma.
\end{align*}

Thus, $\EE[\max f_{2^n1}-\max f_{2^n0}]\leq \sqrt{n\beta/2^n}\sigma$. Meanwhile, we have that $\Var[\max f_{2^n1}-\max f_{2^n0}]\leq \Var[\max f_{2^n1}] + \Var[\max f_{2^n0}]\leq 2(\beta/2^n)\sigma^2$. 
Let $z$ be a random variable subject to $\EE[z]=\sqrt{n\beta/2^n}\sigma$, $\Var[z]=2(\beta/2^n)\sigma^2$, then there exists a $z$ such that $\EE[\max(0, \max f_{2^n1}-\max f_{2^n0})]\leq \EE[\max(0,z)]$. We will bound this value. Denote $c=\sqrt{\beta/2^n}\sigma$ for simplicity. Then we have
\begin{align*}
\exp(\frac{1}{c}\EE[\max(z,0)]) & \leq \EE[\exp(\frac{1}{c}\max(z,0))] \\
& \leq \EE[\max(\exp(\frac{z}{c}), \exp(0))] \\
& \leq \EE[\exp(\frac{z}{c}+1)] \\
& \leq \exp(\sqrt{n}+1)+1.
\end{align*}
Consequently, 
\[
\EE[\max(z,0)] \leq (\sqrt{n}+1+\frac{1}{\exp(\sqrt{n}+1)})\sqrt{\beta/2^n}\sigma.
\]

Hence we have 
\[
\mu_{n+1}-\mu_n=\EE[\max(0, \max f_{2^n1}-\max f_{2^n0})]\leq (\sqrt{n}+1+\frac{1}{\exp(\sqrt{n}+1)})\sqrt{\beta/2^n}\sigma.
\]
By induction, we have $\mu_n \leq \sqrt{\beta/\pi}\sigma+\sum_{i=0}^\infty (\sqrt{i}+1+\frac{1}{\exp(\sqrt{i}+1)})\sqrt{\beta/2^i}\sigma < 8.68\sqrt{\beta}\sigma$, for any integer $n$. Since the dyadic rational $\cup_{i=0}^{\infty} f_{2^i0}$ is dense and compact on $[0,1]$ and $f$ is continuous with probability one, $\EE[\max f]$ shares the same upper bound of $\mu_n$ almost surely. It is shown in Theorem 3 of \cite{lalley2018gp}, that if the expectation $\EE[\max f]$ is bounded then $\max f$ is sub-Gaussian. The lemma follows.
%we have $\EE[\max f_{2^n1}]\leq \mu_n + \sqrt{2n\beta\ln 2/2^n}\sigma$. Thus, $\mu_{n+1}-\mu_n \leq 2(\EE[\max f_{2^n1}]-\mu_n)\leq \sqrt{8n\beta\ln 2/2^n}\sigma$. 
%For any $z_0<0$, $p>0$ such that $\PP(z\leq z_0)\geq p$, the variance of $z$ will be at least $(p/(1-p))(\sqrt{n}c-z_0)^2$, attained when $z\leq (\sqrt{n}c+z_0p)/(1-p)$ with probability one. Hence the term $(p/(1-p))(\sqrt{n}c-z_0)^2$ must be no larger than $\Var[z]=2c^2$, which indicates that $p\leq 2c^2/((\sqrt{n}c-z_0)^2+2c^2)$.
%we replace $\frac{\exp(-\beta_{2^n})}{1+\exp(-2\beta_{2^n})}(f(x_{2i-2})+f(x_{2i}))$ by $\mu_n$ and $\frac{1-\exp(-2\beta_{2^n})}{1+\exp(-2\beta_{2^n})}$ by $\beta_{2^n}$, respectively. Therefore, with $f(x_{2i-1})\sim \NN(\mu_n, \sigma^2\beta_{2^n})$ i.i.d for each $i$, we have $\EE[\max f_{2^n1}]\leq \mu_n + \sqrt{2n\beta\ln 2/2^n}\sigma$. Thus, $\mu_{n+1}-\mu_n \leq 2(\EE[\max f_{2^n1}]-\mu_n)\leq \sqrt{8n\beta\ln 2/2^n}\sigma$. By induction we have $\mu_n \leq \sqrt{\beta/\pi}\sigma+\sum_{i=0}^\infty \sqrt{8i\beta\ln 2/2^i}\sigma < 10.34\sqrt{\beta}\sigma$. Given the upperbound of $\mu_n$ and the fact the dyadic rational $\cup_{i=0}^{\infty} f_{2^i0}$ is dense and compact on $[0,1]$, the lemma follows \cite{lalley2018gp}, Theorem 3.
\end{proof}

%As we use the Gaussian process noise in an iterative algorithm, it is important to discuss the composability of the differential privacy guarantee. \cite{kairouz2013composition} shows that the key of the composability is to analyze the moment generating function of the noise distribution. 
The following Lemma \ref{momentgenerator} shows an upper bound of the moment generating function of the Gaussian process with the kernel introduced in Lemma \ref{clm:sobolev}. The lemma will be used to compose the Gaussian process mechanisms in Theorem \ref{thm:dpql}.

\begin{lemma}
\label{momentgenerator}
Let $g\in H^1$ be continuous almost everywhere. Denote $\PP_0(f)$ and $\PP_1(f)$ be the probability measure over $H^1$ of the sample path generated by $\GG(0,\sigma^2 K)$ and $\GG(g,\sigma^2 K)$, respectively. The sample path $f\sim\PP_0$ satisfies, for any $\lambda>0$ and any set $S$ of sample paths,
\begin{equation*}
\EE_{0}[(\PP_1(f\in S)/\PP_0(f\in S))^\lambda]\leq \exp((\lambda^2+\lambda)\|g\|_{\HH}^2/2\sigma^2).
\end{equation*}
\end{lemma}

\begin{proof}

%We prove the lemma by investigating a dyadic rational. We first sampling from $[0,1]$ the set $\{x_i\}_{2n+1}$, where $x_{i}=i/2n$, $i=0,\dots, 2n$. We define $f_{n0}=(f(x_0), f(x_2),\dots, f(x_{2n}))^T$ and $f_{n1}=(f(x_1), f(x_3),\dots f(x_{2n-1}))^T$ and $\beta_n=\beta/2n$. If $f$ is the sample path from $\PP_0$ by definition, 

%According to the property of multivariate Gaussian variables, the conditional probability of $f_{n1}$ conditioned on $f_{n0}$ follows
%\begin{equation*}
%f_{n1}|f_{n0} \sim \NN(K_{10}K_{00}^{-1}f_{n0}, \sigma^2(K_{11}-K_{10}K_{00}^{-1}K_{10}^{T})).
%\end{equation*}
%Otherwise if $f$ is sampled from $\PP_1$, 
%\begin{equation*}
%f_{n1}|f_{n0} \sim \NN(g_{n1}+K_{10}K_{00}^{-1}(f_{n0}-g_{n0}), \sigma^2(K_{11}-K_{10}K_{00}^{-1}K_{10}^{T})), 
%\end{equation*}

%We omit the steps, but it is immediate to verify that

By the conditional distribution $f_{n1}|f_{n0}$ in Claim \ref{clm:basic}, we have the ratio of the probability density
\begin{align*}
\ln\frac{\PP_1(f_{n1}|f_{n0})}{\PP_0(f_{n1}|f_{n0})} & = -\frac{(g_{n1}-K_{10}K_{00}^{-1}g_{n0})^T(K_{11}-K_{10}K_{00}^{-1}K_{10}^{T})^{-1}(g_{n1}-K_{10}K_{00}^{-1}g_{n0})}{2\sigma^2} \\
& \quad + \frac{2(g_{n1}-K_{10}K_{00}^{-1}g_{n0})^T(K_{11}-K_{10}K_{00}^{-1}K_{10}^{T})^{-1}(f_{n1}-K_{10}K_{00}^{-1}f_{n0})}{2\sigma^2} \\
& = -\frac{1+\exp(-2\beta_n)}{2\sigma^2(1-\exp(-2\beta_n))}\sum_{i=1}^n (g(x_{2i-1})-\frac{\exp(-\beta_n)}{1+\exp(-2\beta_n)}(g(x_{2i-2})+g(x_{2i})))^2 \\
& \quad + \frac{1+\exp(-2\beta_n)}{\sigma^2(1-\exp(-2\beta_n))}\sum_{i=1}^n (g(x_{2i-1})-\frac{\exp(-\beta_n)}{1+\exp(-2\beta_n)}(g(x_{2i-2})+g(x_{2i}))) \\
& \quad \times (f(x_{2i-1})-\frac{\exp(-\beta_n)}{1+\exp(-2\beta_n)}(f(x_{2i-2})+f(x_{2i}))).
\end{align*}
Note that in the equation above $f(x_{2i-1})-\frac{\exp(-\beta_n)}{1+\exp(-2\beta_n)}(f(x_{2i-2})+f(x_{2i}))$ and $f(x_{2j-1})-\frac{\exp(-\beta_n)}{1+\exp(-2\beta_n)}(f(x_{2j-2})+f(x_{2j}))$ are i.i.d. with the distribution $\NN(0, \sigma^2\frac{1-\exp(-2\beta_n)}{1+\exp(-2\beta_n)})$. Hence, We have 
\begin{align*}
& \quad \EE_0[\exp(\lambda\ln\frac{\PP_1(f_{n1}|f_{n0})}{\PP_0(f_{n1}|f_{n0})})] \\ & = \exp(-\lambda\frac{1+\exp(-2\beta_n)}{2\sigma^2(1-\exp(-2\beta_n))}\sum_{i=1}^n(g(x_{2i-1})-\frac{\exp(-\beta_n)}{1+\exp(-2\beta_n)}(g(x_{2i-2})+g(x_{2i})))^2) \\
& \quad + \sum_{i=1}^n (\frac{\lambda^2(1+\exp(-2\beta_n))^2}{2\sigma^4(1-\exp(-2\beta_n))^2}(g(x_{2i-1})-\frac{\exp(-\beta_n)}{1+\exp(-2\beta_n)}(g(x_{2i-2})+g(x_{2i})))^2 \\
& \quad \times \sigma^2\frac{1-\exp(-2\beta_n)}{1+\exp(-2\beta_n)}) \\
& = \exp((\lambda^2-\lambda)\frac{1+\exp(-2\beta_n)}{2\sigma^2(1-\exp(-2\beta_n))}\sum_{i=1}^n(g(x_{2i-1})-\frac{\exp(-\beta_n)}{1+\exp(-2\beta_n)}(g(x_{2i-2})+g(x_{2i})))^2).
%& = \exp((\lambda^2-\lambda)\frac{1+\exp(-2\beta_n)}{2\sigma^2(1-\exp(-2\beta_n))}\sum_{i=1}^n(\frac{(1-\exp(-\beta_n))^2}{1+\exp(-2\beta_n)})^2) \\
%& \quad +
\end{align*}
Meanwhile, 
\begin{align*}
\EE_0[\exp(\lambda\ln\frac{\PP_1(f_{10})}{\PP_0(f_{10})})] & = \EE_0[\exp(2\lambda\frac{f(0)g(0)+f(1)g(1)-\exp(-\beta)(f(0)g(1)+f(1)g(0))}{2\sigma^2(1-\exp(-2\beta))} \\
& \quad - \lambda\frac{g(0)^2+g(1)^2-2\exp(-\beta)g(0)g(1)}{2\sigma^2(1-\exp(-2\beta))})] \\
& = \EE_0[\exp(2\lambda\frac{(g(0)-\exp(-\beta)g(1))f(0)+(g(1)-\exp(-\beta)g(0))f(1)}{2\sigma^2(1-\exp(-2\beta))} \\
& \quad - \lambda\frac{g(0)^2+g(1)^2-2\exp(-\beta)g(0)g(1)}{2\sigma^2(1-\exp(-2\beta))})] \\
& = \exp(\frac{4\lambda^2}{8\sigma^4(1-\exp(-2\beta))^2}\Var{((g(0)-\exp(-\beta)g(1))f(0)+(g(1)-\exp(-\beta)g(0))f(1))}) \\
& \quad - \lambda\frac{g(0)^2+g(1)^2-2\exp(-\beta)g(0)g(1)}{2\sigma^2(1-\exp(-2\beta))}) \\
& = \exp(\frac{4\lambda^2}{8\sigma^4(1-\exp(-2\beta))^2}(\sigma^2(g(0)-\exp(-\beta)g(1))^2+\sigma^2(g(1)-\exp(-\beta)g(0))^2 \\
& \quad + 2\exp(-\beta)\sigma^2(g(0)-\exp(-\beta)g(1))(g(1)-\exp(-\beta)g(0))) \\
& \quad - \lambda\frac{g(0)^2+g(1)^2-2\exp(-\beta)g(0)g(1)}{2\sigma^2(1-\exp(-2\beta))}) \\
& = \exp((\lambda^2-\lambda)\frac{g(0)^2+g(1)^2-2\exp(-\beta)g(0)g(1)}{2\sigma^2(1-\exp(-2\beta))})). \\
\end{align*}
Finally, with $(\varheart)$ follows by cancelling the $g(i-1)g(i)$ terms, $(\vardiamond)$ by relaxing the exponential terms, and $z(i)$ indicating number of bits before the $i$'s last $1$-bit, we have 
\begin{align*}
& \EE_0[\exp(\lambda\ln\frac{\PP_1(f_{2^n0})}{\PP_0(f_{2^n0})})] = \EE_0(\exp(\lambda\ln\frac{\PP_1(f_{2^{n-1}1}|f_{2^{n-1}0})\cdot\dots\cdot\times\PP_1(f_{2^01}|f_{2^00})\PP_1(f_{2^00})}{\PP_0(f_{2^{n-1}1}|f_{2^{n-1}0})\cdot\dots\cdot\times\PP_0(f_{2^01}|f_{2^00})\PP_0(f_{2^00})})) \\ 
& = \EE_0[\exp(\lambda\ln\frac{\PP_1(f_{10})}{\PP_0(f_{10})})]\prod_{k=0}^{n-1}\EE_0[\exp(\lambda\ln\frac{\PP_1(f_{2^k1}|f_{2^k0})}{\PP_0(f_{2^k1}|f_{2^k0})})] \\
& = \exp(\frac{\lambda^2-\lambda}{2\sigma^2}\sum_{k=0}^{n-1}\sum_{i=1}^{2^k}\frac{1+\exp(-2\beta/2^k)}{1-\exp(-2\beta/2^k)}(g((2i-1)2^{-(k+1)}) \\
& \quad -\frac{\exp(-\beta/2^k)}{1+\exp(-2\beta/2^k)}(g((2i-2)2^{-(k+1)})+g((2i)2^{-(k+1)}))^2 \\
& \quad + \frac{(\lambda^2-\lambda)}{2\sigma^2}\frac{1}{1-\exp(-2\beta)}(g(0)^2+g(1)^2-2\exp(-\beta)g(0)g(1))) \\
& \stackrel{(\varheart)}{=} \exp(\frac{\lambda^2-\lambda}{2\sigma^2}(\frac{g(0)^2+g(1)^2}{1-\exp(-2\beta)}+\sum_{k=0}^{n-1}\sum_{i=1}^{2^k} \frac{1+\exp(-2\beta/2^k)}{1-\exp(-2\beta/2^k)}(g((2i-1)2^{-(k+1)})^2 \\
& \quad + \frac{\exp(-2\beta/2^k)}{(1+\exp(-2\beta/2^k))^2}(g((2i-2)2^{-(k+1)})^2+g((2i)2^{-(k+1)})^2)) \\
& \quad - \sum_{i=1}^{2^{n}}\frac{2\exp(-\beta/2^{n-1})}{1-\exp(-2\beta/2^{n-1})}g((i-1)2^{-n})g(i2^{-n}))) \\
& = \exp(\frac{\lambda^2-\lambda}{2\sigma^2} (\sum_{i=0}^{2^n} (\frac{1+\exp(-2\beta/2^{z(i)})}{1-\exp(-2\beta/2^{z(i)})}+2\sum_{k=z(i)+1}^{n-1}\frac{\exp(-2\beta/2^k)}{1-\exp(-4\beta/2^k)}-\frac{\exp(-\beta/2^{n-1})}{1-\exp(-2\beta/2^{n-1})})g(i2^{-n})^2 \\
& \quad + \sum_{i=1}^{2^{n}}\frac{\exp(-\beta/2^{n-1})}{1-\exp(-2\beta/2^{n-1})}g(((i-1)2^{-n})-g(i2^{-n}))^2)+\frac{g(0)^2+g(1)^2}{1-\exp(-2\beta)}) \\
& \stackrel{(\vardiamond)}{\leq} \exp(\frac{\lambda^2-\lambda}{2\sigma^2} (\sum_{i=0}^{2^n} (\frac{1+\exp(-2\beta/2^{z(i)})}{1-\exp(-2\beta/2^{z(i)})}+\sum_{k=z(i)+1}^{n-1}2^k/2\beta-2^{n-1}/2\beta)g(i2^{-n})^2 \\
& \quad + 2^{n}/2\beta \times (g(((i-1)2^{-n})-g(i2^{-n}))^2)+\frac{1}{2}(g(0)^2+g(1)^2)) \\
& \leq \exp(\frac{\lambda^2-\lambda}{2\sigma^2} (\sum_{i=0}^{2^n}\beta^22^{-n}g(i2^{-n})^2+\frac{1}{2\beta}2^{-n}((g((i-1)2^{-n})-g(i2^{-n}))/2^{-n})^2+\frac{1}{2}(g(0)^2+g(1)^2)).
\end{align*}
The lemma follows immediately by letting $\lim n\rightarrow\infty$.
\end{proof}

We restate Proposition \ref{efficiency} and prove it. Recall that $g$ is the sample path and $\hat g$ the linked list to estimate its evaluations.

\begin{namedtheorem}[Proposition \ref{efficiency}]
The noised value function (during either training or released) in Algorithm~\ref{alg:dpql} can respond to $N_q$ queries in $\OO(N_q\ln(N_q))$ time.
\end{namedtheorem}

\begin{proof}
The value function $Q(\cdot)$ is deterministic. It amounts to show that there is an approach to efficiently estimate the Gaussian process sample path $g$.

We consider the $n$-th query, where the previous $n-1$ queries have been computed and stored. Let $x_1,\dots, x_{n-1}$ be the previous queries and $g(x_1),\dots, g(x_{n-1})$ the known value on the sample path. When $x_1\leq, \dots, \leq x_{n-1}$, $K_{10}K_{00}^{-1}$ has only two non-zero elements in each row, where the two elements are consecutive. This property holds per the computation in Claim \ref{clm:basic}, even if $x_i\neq i/2n$. In this case, the two elements need not to be one, but other elements must be zero.

Therefore, the mean $\mu_{at}=K_{10}K_{00}^{-1}$ and the variance $d_{at}=K_{11}-K_{10}K_{00}^{-1}K_{10}^{T}$ can be computed using these two elements in constant time. The exact calculation is shown below. Thus the noised value function can be calculated in $\OO(\ln(N_q))$ time, which is the time complexity of inserting $x_n$ into sorted list $x_1\leq, \dots, \leq x_{n-1}$. The proposition follows.

They exact value of $\mu_{at}$ and $d_{at}$ can be verified immediately so we omit the steps of the derivation. Denote, in the linked list $\hat{g}$ in the algorithm, $s^+$ as the element $s$ links to and $s^-$ as the element that links to $s$. Treat $s^+=1$ and $s^-=0$ for non-existence. When $\hat{g}_k[b]=(s,z)$, denote $\hat{g}_k(s)=z$. Using the arguments above, we have
\small
\begin{align}
\label{eqn:efficient}
\begin{split}
\mu_{at} & = \frac{(\exp(\beta(s-s^-))-\exp(-\beta(s-s^-)))\hat{g}(s^-)}{\exp(\beta(s^+-s^-))-\exp(-\beta(s^+-s^-))} + \frac{(\exp(\beta(s^+-s))-\exp(-\beta(s^+-s)))\hat{g}(s^+)}{\exp(\beta(s^+-s^-))-\exp(-\beta(s^+-s^-))} \\
d_{at} & = -\frac{(\exp(\beta(s-s^-))-\exp(-\beta(s-s^-)))\exp(\beta(s-s^-))}{\exp(\beta(s^+-s^-))-\exp(-\beta(s^+-s^-))} \\
& \quad - \frac{(\exp(\beta(s^+-s))-\exp(-\beta(s^+-s)))\exp(\beta(s^+-s))}{\exp(\beta(s^+-s^-))-\exp(-\beta(s^+-s^-))}+1.
\end{split}
\end{align}
\normalsize
\end{proof}

\begin{claim}
\label{clm:basic}
The following equations hold.
\begin{equation*}
K_{11} = 
\begin{bmatrix}
1 & \exp(-2\beta_n) & \cdots & \exp(-(2n-2)\beta_n) \\
\exp(-2\beta_n) & 1 & \cdots & \exp(-(2n-4)\beta_n) \\
\vdots & \vdots & \ddots & \vdots \\
\exp(-(2n-2)\beta_n) & \exp(-(2n-4)\beta_n) & \cdots & 1
\end{bmatrix},
\end{equation*}
\begin{equation*}
K_{10} = 
\begin{bmatrix}
\exp(-\beta_n) & \exp(-\beta_n) & \cdots & \exp(-(2n-1)\beta_n) \\
\exp(-3\beta_n) & \exp(-\beta_n) & \cdots & \exp(-(2n-3)\beta_n) \\
\vdots & \vdots & \ddots & \vdots \\
\exp(-(2n-1)\beta_n) & \exp(-(2n-3)\beta_n) & \cdots & \exp(-\beta_n)
\end{bmatrix},
\end{equation*}
\begin{equation*}
K_{00} = 
\begin{bmatrix}
1 & \exp(-2\beta_n) & \cdots & \exp(-2n\beta_n) \\
\exp(-2\beta_n) & 1 & \cdots & \exp(-(2n-2)\beta_n) \\
\vdots & \vdots & \ddots & \vdots \\
\exp(-2n\beta_n) & \exp(-(2n-2)\beta_n) & \cdots & 1
\end{bmatrix}.
\end{equation*}
\small
\begin{equation*}
K_{00}^{-1} = \frac{1}{1-\exp(-4\beta_n)}
\begin{bmatrix}
1 & -\exp(-2\beta_n) & 0 & \cdots & 0 \\
-\exp(-2\beta_n) & 1+\exp(-4\beta_n) & -\exp(-2\beta_n) & \cdots & 0\\
0 & -\exp(-2\beta_n) & 1+\exp(-4\beta_n) & \cdots & 0\\
\vdots & \vdots & \vdots & \ddots & \vdots \\
0 & 0 & 0 & \cdots & -\exp(-2\beta_n) \\
0 & 0 & 0 & \cdots & 1
\end{bmatrix},
\end{equation*}
\begin{equation*}
K_{10}K_{00}^{-1} = \frac{\exp(-\beta_n)}{1+\exp(-2\beta_n)}
\begin{bmatrix}
1 & 1 & 0 & \cdots & 0 & 0 \\
0 & 1 & 1 & \cdots & 0 & 0 \\
\vdots & \vdots & \vdots & \ddots & \vdots & \vdots \\
0 & 0 & 0 & \cdots & 1 & 0 \\
0 & 0 & 0 & \cdots & 1 & 1
\end{bmatrix},
\end{equation*}
\normalsize
\begin{align*}
& K_{10}K_{00}^{-1}K_{10}^T = \frac{\exp(-\beta_n)}{1+\exp(-2\beta_n)}\cdot \\
& 
\begin{bmatrix}
2z(1) & z(1)+z(3) & \cdots & z(2n-1)+z(2n-3) \\
z(1)+z(3) & 2z(1) & \cdots & z(2n-3)+z(2n-5) \\
\vdots & \vdots & \ddots & \vdots \\
z(2n-3)+z(2n-5) & z(2n-5)+z(2n-7) & \cdots & z(1)+z(3) \\
z(2n-1)+z(2n-3) & z(2n-3)+z(2n-5) & \cdots & 2z(1)
\end{bmatrix},
\end{align*}
where we write $z(x)=\exp(-x\beta_n)$ for simplicity.
\normalsize
\begin{equation*}
K_{11}-K_{10}K_{00}^{-1}K_{10}^T=\frac{1-\exp(-2\beta_n)}{1+\exp(-2\beta_n)}\II.
\end{equation*}
\end{claim}

\begin{proof}
The claim may not be obvious, but it can be verified immediately.
\end{proof}

\section{Proof of Proposition \ref{clm:hall}}
\label{appendix:others}

The following claim improves Proposition 3 of \cite{hall2013differential}, from the constant $2$ to constant $1.25$. The claim investigates the anisotropic Gaussian noise mechanism, which can be regarded as a discrete version of the Gaussian process mechanism.

\begin{claim}
\label{lem:nonisotropicgaussianmechanism}
%If we define $\Delta=\max_{x,x^\prime}\|M^{-1/2}(f(x)-f(x^\prime))\|_2$ to be the sensitivity under the Mahalanobis distance 
Let $f$ be an vector-input vector-output function. Define the sensitivity under the Mahalanobis distance as $\Delta=\max_{x,x^\prime}\|M^{-1/2}(f(x)-f(x^\prime))\|_2$ where $M$ is positive definite symmetric. Then if $0<\epsilon<1$ and $\sigma\geq \sqrt{2\ln(1.25/\delta)}\Delta/\epsilon$, $f(x)+\sigma M^{1/2}y$ is ($\epsilon$,$\delta$)-DP, where $y$ is drawn from $\NN(0,\sigma^2 I)$.
\end{claim}

\begin{proof}
Let $z\in\RR^d$ and $c=\ln (\PP(f(x)+\sigma M^{1/2}y=z)/\PP(f(x^\prime)+\sigma M^{1/2}y=z))$, 
\begin{align*}
c(z) & =\ln \frac{\PP(f(x)=z)}{\PP(f(x^\prime)=z)} \\ 
& = \frac{(z-f(x))^TM^{-1}(z-f(x))}{2\sigma^2} - \frac{(z-f(x^\prime))^TM^{-1}(z-f(x^\prime))}{2\sigma^2} \\
& = \frac{\|f(x)-f(x^\prime)\|_M^2-2y^TM^{-1/2}(f(x)-f(x^\prime))}{2\sigma^2}.
\end{align*}
Hence, when $y\sim \NN(0,1)$,
$$ c(z) \sim \NN(\|f(x)-f(x^\prime)\|_M^2/2\sigma^2, 2\|f(x)-f(x^\prime)\|_M^2/2\sigma^2).$$
The rest of the argument follows the approach in~\cite{dwork2014algorithmic} page 261, which is described in the setting of one-dimensional random variables and isotropic $M$. We show the argument in our setting for completeness.
For $\delta$-approximation privacy we would like to have $\PP(c<\epsilon)>1-\delta/2$. We consider the following tail bound of the Gaussian distribution: $\forall t$,
$$\PP(c\geq \EE[c]+t) \leq \exp(-t^2/2\Var(c))\sqrt{\Var(c)/2t^2\pi},$$
which indicates that it is sufficient if both $\ln(\sqrt{2/\pi\delta^2})\leq\ln(t\sigma/\|f(x)-f(x^\prime)\|_M)+t^2\sigma^2/2\|f(x)-f(x^\prime)\|_M^2$ and $\|f(x)-f(x^\prime)\|_M^2/2\sigma^2+t\leq \epsilon$ are satisfied for some $t$. The conditions are further reduced to $\ln(\sqrt{2/\pi\delta^2})\leq\ln(t\sigma/\Delta)+t^2\sigma^2/2\Delta^2$ and $t\leq \epsilon-\Delta^2/2\sigma^2$, respectively. We insert $t=\epsilon-\Delta^2/2\sigma^2$ to the first inequality and derive:
\begin{equation*}
\ln(\frac{\epsilon\sigma}{\Delta}-\frac{\Delta}{2\sigma})+(\frac{\epsilon^2\sigma^2}{2\Delta^2}+\frac{\Delta^2}{8\sigma^2}-\frac{\epsilon}{2})\geq\ln(\sqrt{2/\pi\delta^2}).
\end{equation*}
With $\epsilon\leq 1$ we have $\frac{\epsilon\sigma}{\Delta}-\frac{\Delta}{2\sigma}\geq 1$ whenever $\sigma\epsilon/\Delta \geq 3/2$. With $\sigma\epsilon/\Delta \geq 3/2$ we have $\frac{\epsilon^2\sigma^2}{2\Delta^2}+\frac{\Delta^2}{8\sigma^2}-\frac{\epsilon}{2}\geq \sigma^2\epsilon^2/2\Delta^2-4/9$ per the monotonicity with respect to $\sigma\epsilon/\Delta$. Hence, it is sufficient that both $\sigma\epsilon/\Delta \geq 3/2$ and $\sigma^2\epsilon^2/2\Delta^2-4/9$ are satisfied. The choice $\sigma\geq \sqrt{2\ln(1.25/\delta)}\Delta/\epsilon$ the immediately follows, as desired.
\end{proof}

Now Proposition \ref{clm:hall} of this paper follows the Claim \ref{lem:nonisotropicgaussianmechanism} above, and Proposition 7 and Proposition 8 of \cite{hall2013differential}.

\section{Proofs of Lemma \ref{clm:sobolev} and Corollary \ref{nnlipschitz}}
\label{appendix:sobolev}

Recall that Lemma \ref{clm:sobolev} finds the desired RKHS and kernel and Corollary \ref{nnlipschitz} is an immediate result that neural networks belong to this RKHS. The key observation is that we cannot restrict the value of $f(0)$ and $f(1)$ to be zero (which is common in some analysis of RKHS), as the value function should not be assumed to have zero value at the boundary.

\begin{namedtheorem}[Lemma \ref{clm:sobolev}]
We consider the one-dimensional function with bounded variable $x\in \RR$, which, without loss of generality, can be treated as $x\in [0,1]$. The Sobolev space $H^1$ with order $1$ and the $\ell^2$-norm (also written as $W^{1,2}$ conventionally) is defined as
\begin{equation*}
H^1 = \{f\in C[0,1]: \partial f(x) \text{ exists}; \int_0^1 (\partial f(x))^2 dx < \infty \},
\end{equation*}
where $\partial f(x)$ denotes weak derivatives and $\int (\cdot)dx$ denotes the Lebesgue integration. If $H^1$ is equipped with inner product
\begin{align*}
\langle f,g\rangle & = \frac{1}{2}(f(0)g(0)+f(1)g(1)) + \frac{1}{2\beta}\int_0^1 \partial f(x)\partial g(x) + \beta^2 f(x)g(x) dx
\end{align*}
where $\beta>0$, it is an RKHS with kernel $K(x,y)=\exp(-\beta |x-y|)$.
\end{namedtheorem}

Note that a function being differentiable almost everywhere does not imply that the function has a weak derivative. A counterexample is the Cantor function, which is thus excluded from $H^1$. Any function equal to $\partial f(x)$ almost everywhere is considered identical to $\partial f(x)$ in $H^1$.

\begin{proof}
The Sobolev space defined in the lemma does not constrain the value to be zero on its border $\{0,1\}$, hence standard arguments do not apply. However the arguments will be similar. It suffices to show that $f(x)^2\leq c\langle f,f\rangle$ for some $c$ and that $H^1$ is complete. The former can be seen by showing that for any nonzero $f$, $\langle f,g\rangle=0$ if and only if $g=0$. By the Cauchy-Schwartz
inequality, 
\begin{align*}
\langle f,f\rangle & \geq \frac{1}{2\beta} \int_0^1 (\partial f(x))^2 dx \geq \frac{1}{2\beta} \int_0^x (\partial f(z))^2 dz \\
& \geq \frac{1}{2\beta} (\int_0^x \partial f(z) dt)^2 = \frac{1}{2\beta}f(x)^2. 
\end{align*}
We set $c=2\beta$ as desired. 
For the completeness of $H^1$, we show that for any sequence $\{f_n\}$ with $\langle f_n-f_{n+1},f_n-f_{n+1}\rangle$ converging to zero, the limit of the sequence is in $H^1$. In fact, $\langle f_n-f_{n+1},f_n-f_{n+1}\rangle$ converging to zero indicates that $\int_0^1 (f_n-f_{n+1}) dx$ converges to zero, which then indicates that $\{f_n(x)\}$ converges pointwise for any $x$. The first part of the lemma follows.
We then verify that $f(y)=\langle f,K_y\rangle$ for any function $f(y)\in H^1$. In fact, 
\begin{align*}
\langle f,K_y\rangle & = \frac{1}{2}(f(0)K_y(0)+f(1)K_y(1)) + \frac{\beta}{2} \int_0^1 f(x)K_y(x) dx \\
& \quad + \frac{1}{2\beta} (f(x)\partial K_y(x)\big\rvert_0^1 - \int_0^1 f(x)\partial^2 K_y(x) dx) \\
& = \frac{1}{2}(f(0)K_y(0)+f(1)K_y(1)) + \frac{\beta}{2} \int_0^1 f(x)K_y(x) dx + \frac{1}{2\beta} f(x) (-\beta u(x-y)K_y(x))\big\rvert_0^1 \\
& \quad - \frac{1}{2\beta}\int_0^1 f(x)((-\beta u(x-y))^2K_y(x) - 2\beta\delta(x)K_y(x)) dx) \\
& = -\frac{1}{2\beta}\int_0^1 -f(x)2\beta\delta(x-y)K_y(x) dx = f(y),
\end{align*}
where $u(x)$ is the sign function and $\delta(x)$ is the impulse function. By the Riesz representation theorem, $K(x,y)$ is the unique kernel of $H^1$ equipped with the inner product $\langle f,g\rangle$ defined above. 
\end{proof}

\begin{namedtheorem}[Corollary \ref{nnlipschitz}]
Let $\hat{f}_W(x)$ denote the neural network with finite many finite parameters $W$. For $\hat{f}_W(x)$ with finite many layers, if the gradient of the activation function is bounded, then $\hat{f}_W(x)\in H^1$.
%$\hat{f}_W(x)$ is thus $L$-Lipschitz continuous.
\end{namedtheorem}

\begin{proof}
Let $\psi_i(\cdot)$ be the activation function so the $i$-th layer is represented by the function $\hat{f}_i(x)=\psi(w_ix+b_i)$. Let the gradient of $\psi_i(\cdot)$ be bounded by $c$.
%and $c$ be the bound of $|\partial \psi(x)|$. We take the ReLU function $\psi(x)=\min(0, x-0.5)$ (without loss of generality we map the domain to $[0,1]$) as an example where $c=1$. Let $\hat{f}_i$ denote the function of each layer such that $\hat{f}_i(x)=\psi(w_ix+b_i)$. 
As per the chain rule we have 
%|\partial \hat{f}_W(x)|\leq c^{N_c}\prod_i |w_i|$ which is also bounded, where $N_c$ is the number of layers. Then, 
\begin{equation*}
\int_0^1 (\partial \hat{f}_W(x))^2 dx \leq \prod_{i=1}^{N_c} (f_1(\dots f_{i-1}(x))\partial f_i|_{x=x_i})^2 < \infty
\end{equation*}
for some $x_i$. Hence $\hat{f}_W(x)\in H^1$.
%We have $L\leq |\partial \hat{f}(x)|\leq c^{N_c}\prod_i |w_i|$ by its definition.
\end{proof}

\section{Proof of Proposition \ref{utility}}
\label{sec:utility}

As the proposition is under the finite state space setting, the context in this section will be different from the rest of the paper. We first restate the proposition and define the notations and preliminaries needed in the analysis. We then show some intermediate claims before we give the final proof.

In the discrete state space setting, we have $\cS= \{1,\dots,n\}$. The stochastic transition kernel is the probability distribution $\PP(s^\prime|s,a)$, denoted as the matrices $P_a\in \RR^{n\times n}$ (each row sums up to one), $a=1,\dots,m$. We write the reward function as $r_a\in\RR^n$, $a=1,\dots,m$. In this setting, finding the optimal action-state value function is equivalent to finding the optimal state value function $v(s)$, denoted as a vector $v\in \RR^n$. The Bellman equation for the optimal value function is given by
\begin{equation}
\label{eqn:bellmandiscrete}
v \geq \gamma P_iv+r_i,
\end{equation}
for each $i=1,\dots,m$.

\begin{namedtheorem}[Proposition \ref{utility}]
Let $v^\prime$ and $v^\ast$ be the value function learned by our algorithm and the optimal value function, respectively. In the case $J=1$, $|S|=n<\infty$, and $\gamma<1$, the utility loss of the algorithm satisfies
\begin{equation*}
\EE[\frac{1}{n}\|v^\prime-v^\ast\|_1]\leq \frac{2\sqrt{2}\sigma}{\sqrt{n\pi}(1-\gamma)}.
\end{equation*}
\end{namedtheorem}

Our utility analysis is based on the linear program formulation under discrete state spaces \cite{de2002linear,chen2016stochastic} and the sensitivity of linear programs \cite{hsu2014privately}. In the discrete setting, $v$ is optimal if and only if the Bellman equation \eqref{eqn:bellmandiscrete} is satisfied.
In fact, the `if' relation is immediate, and the `only if' relation is shown in \cite{sutton2018reinforcement} page 64.
By exhausting the action set under the $\max$ operator and numbering the actions from $1$ to $m$, the Bellman equation is formulated into the below linear program:
\begin{equation}
\begin{aligned}
& \underset{v}{\mathrm{minimize}}
& & \ee^Tv \\
& \mathrm{subject}\text{ }\mathrm{to}
& & (\II-\gamma P_i)v-r_i \geq 0, \quad i=1,\dots,m,
\label{eqn:appre-prime}
\end{aligned}
\end{equation}
where $\ee$ is the all-one vector and $\ee^Tv$ is the dummy objective. The dual of the linear program \eqref{eqn:appre-prime} is
\begin{equation*}
\begin{aligned}
& \underset{\lambda_1,\dots,\lambda_m}{\mathrm{maximize}}
& & \sum_i\lambda_i^Tr_i \\
& \mathrm{subject}\text{ }\mathrm{to}
& & \sum_i(\II-\gamma P_i^T)\lambda_i = \ee, \\
& & & \lambda_i \geq 0, \quad i=1,\dots,m.
\label{eqn:appre-dual}
\end{aligned}
\end{equation*}

We consider the discrete version of Algorithm \ref{alg:dpql}. The Gaussian process noise degenerates to multivariate Gaussian noise. It is also observed that adding noise to the value function is equivalent to adding noise to the reward function, as they are additive in the update. With $J=1$, it uses the same sample of noise through the training process. 

The convergence of the algorithm is guaranteed. In fact, per \cite{sutton2018reinforcement} Section 4.4, the value iteration algorithm will converge to the optimal value function of the noised reward function. Formally, given the transition matrices and the noised reward signal $r_i^\prime=r_i+z_i$ for $i=1,\dots,m$ where $z_i\sim \NN(0, \sigma^2 \II)$, Algorithm \ref{alg:dpql} is guaranteed to converge to a value function $v^\prime$. We desire to show that  
\[
\EE[\frac{1}{n}\|v^\prime-v^\ast\|_1]\leq \frac{2\sqrt{2}\sigma}{\sqrt{n\pi}(1-\gamma)},
\]
where $v^\prime$ and $v^\ast$ are the optimal value function under the reward $r^\prime$ and $r$, respectively. $v^\prime$ and $v^\ast$ are therefore the solution of the system \eqref{eqn:appre-prime} under the reward signal $r^\prime$ and $r$, respectively.

\begin{lemma}[\cite{de2002linear} and \cite{chen2016stochastic}]
\label{claim:existence}
There exists an optimal dual solution $\lambda_i^\ast$, $i=1,\dots,m$, an optimal deterministic policy $\pi^\ast(\cdot)$, and the corresponding transition matrix $P^\ast$, such that
$$\sum_i\lambda_{i}^\ast=(\II-\gamma P^{\ast T})^{-1}\ee,$$
and the $k$-th entry of $\lambda_i^\ast$ equals to the $k$-th entry of $\sum_i\lambda_{i}^\ast$ if $\pi^\ast(k)=i$, and zero otherwise.
\end{lemma}

\begin{proof}
Similar proofs are presented in \cite{de2002linear} and \cite{chen2016stochastic}. For the completeness of our paper we prove the claim under our context and notations. Denote the superscript $(k)$ as the $k$-th element for a vector and as the $k$-th row for a matrix. Specify $\xi_i^{\ast}$ to be any dual optimal solution and construct the policy $\pi^\ast(k)=\argmax_i\xi_i^{\ast (k)}$. Then let 
\[
\lambda^\ast=(\II-\gamma P^{\ast T})^{-1}\ee,
\]
where $P^\ast$ is the transition matrix of $\pi^\ast(\cdot)$. The inversion exists since all the eigenvalues of the Markov matrix $P^*$ are smaller than one. Define $\lambda_i^\ast$, $i=1,\dots,m$, such that $\lambda_i^{\ast(k)}=\lambda^{\ast(k)}$ whenever $\pi^\ast(k)=i$ and zero otherwise. We have for $\lambda_i^{\ast}$ that 
\[
\sum_k\sum_i\lambda_i^{(k)}(\II-\gamma P_i)^{(k)}=\ee,
\]
which is a rewrite of the dual feasibility by summing over $k$. We also have $\lambda_i^{*(k)}=0$ whenever $\xi_i^{\ast(k)}=0$ for any $i$ and $k$, and together with the slackness 
\[
{\xi_i^{\ast}}^T((\II-\gamma P_i)v-r_i)=0,
\]
we have ${\lambda_i^{\ast}}^T((\II-\gamma P_i)v-r_i)=0$. The optimality of $\lambda_i^{\ast}$, $i=1,\dots,m$ follows.
\end{proof}

\begin{claim}
The $\ell^1$-norm of the dual optima $\|\sum_i\lambda_{i}^\ast\|_1$ is exactly $n/(1-\gamma)$.
\end{claim}

\begin{proof}
By definition we have $\|\sum_i\lambda_{i}^\ast\|_1=\|\lambda^\ast\|_1$ and $(\II-\gamma {P^\ast}^T)\lambda^\ast=\ee$. Since $P^{\ast}$ is a Markov matrix, we have $\|{P^\ast}^T\lambda^\ast\|_1=\|\lambda^\ast\|_1$. Taking $\ell^1$-norm and we have $\|\lambda^\ast\|_1-\gamma\|\lambda^\ast\|_1=\|\ee\|_1$. The claim follows.
\end{proof}

The following lemma justifies that there exists an algorithm to attain bounded suboptimality, given only the noised reward signal. We regard this property as the robustness of the dual system.

\begin{lemma}
\label{claim:dualbound}
Let $\lambda_i^\prime$, $i=1,\dots,m$, be the optimal solution of the system
\begin{equation}
\begin{aligned}
& \underset{\lambda_1,\dots,\lambda_m}{\mathrm{maximize}}
& & \sum_i\lambda_i^Tr_i^\prime \\
& \mathrm{subject}\text{ }\mathrm{to}
& & \sum_i(\II-\gamma P_i^T)\lambda_i = \ee, \\
& & & \sum_i\ee^T\lambda_i \leq \frac{n}{1-\gamma}, \\
& & & \lambda_i \geq 0, \quad i=1,\dots,m,
\label{eqn:appre-dualperturb}
\end{aligned}
\end{equation}
we have 
\begin{equation*}
%\EE[\sum_i\lambda_i^{\prime T}\bar{r}_i]\geq((\ee^T{v^\ast})-2\alpha n/(1-\gamma))(1-2\beta)^{nm},    
\EE[\sum_i\lambda_i^{\prime T}r_i] \geq \sum_i\lambda_i^{\ast T}r_i - \frac{2\sqrt{2}n\sigma}{\sqrt{\pi}(1-\gamma)}.
\end{equation*}
%for any $\alpha$ where $\beta=\frac{\sigma^2}{\alpha\sqrt{2\pi}}\exp(-\alpha^2/2\sigma^2)$.
\end{lemma}

\begin{proof}
%for each element in $z_i$ we have the tail bound that with probability at least $1-2\beta$, that element has an absolute value no greater than $\alpha$. As those elements are independent, the union bound indicates that with probability at least $(1-2\beta)^{nm}$ we have $z_i\leq \alpha\ee$ for all $i$. Under $z_i\leq \alpha\ee$, we have
With $(\varheart)$ follows the strong duality and $(\vardiamond)$ follows the non-negativity and the convexity, we have
\begin{align*}
\EE[\sum_i\lambda_i^{\prime T}r_i] & = \EE[\sum_i\lambda_i^{\prime T}(r_i^\prime-z_i)] \\
& \geq \EE[\sum_i\lambda_i^{\ast T}r_i^\prime-\sum_i\lambda_i^{\prime T}z_i] \\
& = \EE[\sum_i\lambda_i^{\ast T}(r_i+z_i)-\sum_i\lambda_i^{\prime T}z_i] \\
& \stackrel{(\varheart)}{=} \sum_i\lambda_i^{\ast T}r_i + \EE[\sum_i(\lambda_i^\ast-\lambda_i^{\prime})^Tz_i] \\
& \stackrel{(\vardiamond)}{\geq} \sum_i\lambda_i^{\ast T}r_i - \frac{2}{m(1-\gamma)}\EE[\sum_i\|z_i\|_1] \\
& = \sum_i\lambda_i^{\ast T}r_i - \frac{2\sqrt{2}n\sigma}{\sqrt{\pi}(1-\gamma)}. \tag*{\qedhere}
\end{align*}
\end{proof}

It suffices to discuss the connection between the robustness of the primal and the robustness of the dual, which will help us to give a rigorous bound of the utility loss.

Intuitively, if we replace the maximum over $\lambda_i$ by the fixed $\lambda_i^\prime$ in the below derivation of the slackness equation, the subsequent equations will yield $\EE[\sum_i\lambda_i^{\prime T}r_i]$ which is desired. We observe that relaxing policy optimization (primal) side of the system at the saddle point results in an infeasible point at the value learning (dual) system. It amounts to show that this infeasible point can be mapped to the set of suboptimal values functions. Let $A$ and $B$ be the optimal value of the primal and the dual, the derivation of the slackness equation can be written as
\begin{align*}
A &= \min_v\max_{\lambda_1,\dots,\lambda_m\geq 0}\  \ee^Tv-(\lambda_1^T((\II-\gamma P_1)v-r_1)+\dots+\lambda_m^T((\II-\gamma P_m)v-r_m)) \label{eqn:saddle}\\
&\geq \max_{\lambda_1,\dots,\lambda_m\geq 0}\min_v\  \ee^Tv-(\lambda_1^T((\II-\gamma P_1)v-r_1)+\dots+\lambda_m^T((\II-\gamma P_m)v-r_m)) \nonumber\\
&= \max_{\lambda_1,\dots,\lambda_m\geq 0}\min_v\  (\lambda_1^Tr_1+\dots\lambda_m^Tr_m)-(-e^T+\lambda_1^T(\II-\gamma P_1)+\dots+\lambda_m^T(\II-\gamma P_m))v \nonumber = B.
\end{align*}

\begin{claim}
\label{claim:bellmandeter}
The stochastic policy $\pi^\prime(i|k)=\lambda_i^{\prime(k)}/\sum_{i^\prime}\lambda_{i^\prime}^{\prime(k)}$ achieves the value $v^\prime$ such that $\ee^Tv^\prime=\sum_i\lambda_i^{\prime T}r_i$.
\end{claim}

\begin{proof}
With Lemma \ref{claim:existence} showing the existence, specify $\lambda^{\prime\prime}=(\II-\gamma P^{\prime\prime T})^{-1}\ee$ and $\lambda_i^{\prime\prime}$ to be the optimal solution of \eqref{eqn:appre-dualperturb} where $P^{\prime\prime}$ is the corresponding transition matrix. The Bellman equation indicates that $((\II-\gamma P_i)v^\prime-r_i)^{(k)}=0$ whenever $\lambda_i^{\prime\prime (k)}>0$. It is equivalent to $(\II-\gamma P^{\prime\prime})v^\prime-\tilde{r}=0$ where $\tilde{r}^{(k)}=r_{\pi(k)}^{(k)}$, $k=1,\dots,n$. Hence, 
\[
\ee^Tv^\prime=\ee^T(\II-\gamma P^{\prime\prime})^{-1}\tilde{r}=\tilde{r}^T(\II-\gamma P^{\prime\prime})^{-1}\ee=\tilde{r}^T\lambda^{\prime\prime}=\sum_i\lambda_i^{\prime T}r_i. \tag*{\qedhere}
\]
%We investigate the transition probability matrix $P_\pi$ under the stated mixed policy. In fact, 
%$$P_\pi=(\frac{\lambda_1^{(1)}P_1^{(1)}+\dots+\lambda_m^{(1)}P_m^{(1)}}{\lambda_1^{(1)}+\dots+\lambda_m^{(1)}}, \dots, \frac{\lambda_1^{(n)}P_1^{(n)}+\dots+\lambda_m^{(n)}P_m^{(n)}}{\lambda_1^{(n)}+\dots+\lambda_m^{(n)}}).$$
%The Bellman equation indicates that $(\II-\gamma P_\pi)v-P_\pi r_i$
%The optimal of $\sum_i\lambda_i$ if the optimal policy is unique can be treated as a single vector and plug into the dual feasibility as $\sum_i\lambda_i=\ee+\gamma P_\ast^T\sum_i\lambda_i$. Try deal with $\sum_i\lambda_i$ with the transition matrix above.
\end{proof}

Armed with the above results, we prove the proposition of the utility guarantee.

\begin{proof}[Proof of Proposition \ref{utility}]
By Lemma \ref{claim:dualbound}, our algorithm finds $\lambda_i^\prime$ by solving \eqref{eqn:appre-dualperturb} which satisfies that $\sum_i\lambda_i^{\ast T}r_i-\EE[\sum_i\lambda_i^{\prime T}r_i]\leq \frac{2\sqrt{2}n\sigma}{\sqrt{\pi}(1-\gamma)}$. By Claim \ref{claim:bellmandeter} we have $\EE[\sum_i\lambda_i^{\prime T}r_i]=\EE[\ee^Tv^\prime]$. The strong duality then suggests $\sum_i\lambda_i^{\ast T}r_i=\ee^Tv^\ast$. As $\EE[\|v^\prime-v^\ast\|_1]=\ee^Tv^\ast-\EE[\ee^Tv^\prime]$, the proposition follows.
\end{proof}

\section{Details of the Experiments}
\label{appendix:experiments}

\subsection{The Environment}
\label{appendix:experiments-env}

The MDP environment is defined as follows: $\cS=[0,1]$ and the state $s$ denotes the location of the agent. $s_0$ is uniformly distributed on $\cS$. $\cA=\{0,1\}$. If the agent chooses action 1, the agent will randomly move towards the right by a random amount sampled uniformly from $[0, 0.25]$. If after the move $s$ is greater than 1, it will be reset to 1. Respectively, if the agent chooses action 0, the agent will randomly move towards the left by a random amount sampled uniformly from $[0, 0.25]$. If after the move $s$ is less than 0, it will be reset to 0. The reward $0.5-|s-0.5|$ is given at each step, which encourages the agent to move close to the middle of the state space. Each episode of the MDP terminates at the 50$^{\text{th}}$ step. The algorithms are trained on $100$ episodes or equivalently 5000 samples. The code is available with this manuscript submission.

\subsection{The Baseline Approaches}
\label{appendix:experiments-baselines}

Balle, Gomrokchi and Precup~\cite{balle2016differentially} consider differentially private policy evaluation, where the value function is learned on a one-step MDP using a linear function approximator. This work protects the reward sequence from being distinguishable, but does not ensure the privacy of newly visited states when the value function is released. Thus we do not consider the work as differentially private under our aim of protecting the reward function. Studies on differentially private contextual bandits by Sajed and Sheffet \cite{sajed2019optimal} and by Shariff and Sheffet \cite{shariff2018differentially} are considering the equivalent problem, while we use \cite{balle2016differentially} to represent these works.

We also compare with the algorithm proposed by Venkitasubramaniam~\cite{venkitasubramaniam2013privacy}, via input perturbation. In the work, every reward signal is protected by a Gaussian noise thus making the algorithm differentially private. The privacy guarantee is straightforwardly derived by the composition theory by Kairouz, Oh, and Viswanath \cite{kairouz2013composition}.

We finally compare our algorithm with the differentially private deep learning by Abadi et al. \cite{abadi2016deep}. As we use a neural network, we can perturb the gradient estimator in the updates such that all inputs are indistinguishable. We use the derived bound in Theorem 1 of \cite{abadi2016deep}. The $c_2$ constant in that theorem is assigned by $\sqrt{2}$ by following the proof of the theorem.

\subsection{Parameters of Our Approach}
\label{appendix:experiments-parameters}

We have demonstrated the algorithms on the target of $\epsilon=0.9$, $\delta=1\cdot 10^{-4}$ for Figure \ref{fig:empirical-compare}(a) and $\epsilon=0.45$, $\delta=1\cdot 10^{-4}$ for Figure {fig:empirical-compare}(b), respectively. In this section we show how exactly these privacy targets are achieved.

Theorem \ref{thm:dpql} indicates that our algorithm is $(\epsilon, \delta+J\exp(-(2k-8.68\sqrt\beta\sigma)^2/2))$-DP when
\begin{equation*}
\sigma\geq\sqrt{2(T/B)\ln(e+\epsilon/\delta)}C(\alpha, k, L, B)/\epsilon,
\end{equation*}
where $C(\alpha, k, L, B)=((4\alpha(k+1)/B)^2+4\alpha(k+1)/B)L^2$, $\beta=(4\alpha (k+1)/B)^{-1}$.
We reset the noise on every iteration, namely, let $J=\lfloor T/B \rfloor$. We rewrite the term $J\exp(-(2k-8.68\sqrt\beta\sigma)^2/2)$ as a tight bound $1-(1-\exp(-(2k-8.68\sqrt\beta\sigma)^2/2))^J$, which is the probability that all $J$ sample paths are bounded by $2k$. Now we derive the set of parameters.

Let $\delta_g=1-(1-\exp(-(2k-8.68\sqrt\beta\sigma)^2/2))^J$ and $v=(4\alpha (k+1)/B)$. Then $\beta=1/v$ and $C\approx vL^2$. Plugging in both the values and $T=5000$, $B=64$ we have $2k-8.68\sqrt\beta\sigma\approx 2k-8.6\sqrt{k+1}$. Similar to \cite{abadi2016deep} we target $1\times 10^{-4}$-approximation, where it is sufficient if $\delta\leq 5\times 10^{-5}$ and $\delta_g\leq 5\times 10^{-5}$. To satisfy $\delta_g\leq 5\times 10^{-5}$ we need $2k-8.6\sqrt{k+1}=\ln(1-\exp(\ln(1-\delta_g)/J))\approx 3.5$, when $J=78$. Thus $k=23$ will be sufficient. Plugging this $k$ value back to $v$ we have $v=6.19\times 10^{-5}$, when $\alpha=3\times 10^{-4}$. Finally we target high-privacy regime $\epsilon=0.9$ and plug in $L^2=16$ and have $\sigma\approx\sqrt{2(T/B)\ln(\epsilon/\delta)}vL^2/\epsilon\approx 0.313$.

Approximations are made in the above arguments, but it is immediate to verify that when $\alpha=3\times 10^{-4}$, $k=2.3$, $L^2=16$, $B=64$, $T=5\times 10^{3}$, and $\sigma=0.32$ the algorithm is $(0.9, 1\times 10^{-4})$-differentially private. When $\sigma=0.74$ the algorithm is $(0.45, 1\times 10^{-4})$-differentially private. The above parameters correspond to Figure \ref{fig:empirical-compare}(a) and \ref{fig:empirical-compare}(b), respectively.

\end{document}